%% file: main_arxiv.tex
\documentclass{article} %
\usepackage{hyperref}
\usepackage{url}
\usepackage[utf8]{inputenc}
\usepackage{booktabs} %
\usepackage{nicefrac} %
\usepackage{microtype} %
\usepackage{graphicx}

\usepackage{natbib}
\usepackage{gensymb}
\usepackage{color}
\usepackage{caption}
\usepackage{subcaption}

\usepackage{amsfonts} %
\usepackage{amsthm}
\usepackage{mathrsfs}
\usepackage{amsmath}
\usepackage{amssymb}
\usepackage{mathtools}
\usepackage{dsfont}
\usepackage{mathrsfs}
\usepackage{bbold}

\usepackage{wrapfig}
\usepackage{soul}
\usepackage[ruled,longend, algo2e]{algorithm2e}
\usepackage[textwidth=1.2in,textsize=tiny]{todonotes}
\usepackage{diagbox}
\usepackage{algorithm}
\usepackage{algorithmic}
\usepackage{hhline}
\usepackage{multirow}

\definecolor{codegreen}{rgb}{0,0.3,0.6}
\definecolor{codegray}{rgb}{0.5,0.5,0.5}
\definecolor{codepurple}{rgb}{0.58,0,0.82}
\definecolor{backcolour}{rgb}{0.95,0.95,0.92}
\definecolor{orange}{rgb}{1,0.5,0}

\usepackage{enumitem}

\usepackage{listings}
\lstdefinestyle{mystyle}{
    basicstyle=\tiny,
    commentstyle=\color{codegreen},
    keywordstyle=\color{magenta},
    numberstyle=\tiny\color{codegray},
    stringstyle=\color{codepurple},
    basicstyle=\fontsize{8.5}{9}\selectfont\ttfamily,
    breakatwhitespace=false,         
    breaklines=true,                 
    captionpos=b,                    
    keepspaces=true,                 
    numbers=none,                    
    numbersep=5pt,                  
    showspaces=false,                
    showstringspaces=false,
}
\usepackage{hhline}
\usepackage{multirow}

\usepackage[
  letterpaper,
  top=1in,
  bottom=1in,
  left=1in,
  right=1in
]{geometry}

\newtheorem{theorem}{Theorem}[section]
\newtheorem{assumption}[theorem]{Assumption}
\newtheorem{definition}[theorem]{Definition}
\newtheorem{lemma}[theorem]{Lemma}
\newtheorem{proposition}[theorem]{Proposition}
\newtheorem{corollary}[theorem]{Corollary}
\newtheorem{remark}[theorem]{Remark}

\newcommand{\Mat}{\boldsymbol}
\newcommand{\Set}{\mathcal}

\newcommand{\real}{\mathbb{R}}
\newcommand{\integer}{\mathbb{N}}
\newcommand{\complex}{\mathbb{C}}

\newcommand{\wh}[1]{\widehat{#1}}

\DeclareMathOperator{\Span}{span}

\DeclareMathOperator{\E}{\mathbb{E}}
\DeclareMathOperator{\mean}{\mathbb{E}}

\DeclareMathOperator{\gauss}{\mathcal{N}}
\DeclareMathOperator{\vol}{\mathrm{vol}}

\DeclareMathOperator*{\argmin}{arg\,min}

\title{Why Neural Network Can Discover Symbolic Structures with Gradient-based Training: An Algebraic and Geometric Foundation for Neurosymbolic Reasoning
}

\renewcommand{\thefootnote}{\fnsymbol{footnote}}

\author{
Peihao Wang\footnotemark[2] \footnotemark[1] \and
Zhangyang ``Atlas'' Wang\footnotemark[2] \footnotemark[1]
}
\date{\today} %

\begin{document}

\maketitle

\footnotetext[1]{Z. Wang developed the initial framework as his pet project, while P. Wang proved core results in this work. As a result, they contributed equally to this paper.}
\footnotetext[2]{Department of Electrical and Computer Engineering, University of Texas at Austin. Email addresses: \texttt{\{peihaowang, atlaswang\}@utexas.edu}.}

\renewcommand{\thefootnote}{\arabic{footnote}}

\begin{abstract}
We develop a theoretical framework that explains how discrete symbolic structures can emerge naturally from continuous neural network training dynamics. By lifting neural parameters to a measure space and modeling training as Wasserstein gradient flow, we show that under geometric constraints, such as group invariance, the parameter measure $\mu_t$ undergoes two concurrent phenomena:  (1) a decoupling of the gradient flow into independent optimization trajectories over some potential functions, and (2) a progressive contraction on the degree of freedom.
These potentials encode algebraic constraints relevant to the task and act as ring homomorphisms under a commutative semi-ring structure on the measure space.
As training progresses, the network transitions from a high-dimensional exploration to compositional representations that comply with algebraic operations and exhibit a lower degree of freedom.
We further establish data scaling laws for realizing symbolic tasks, linking representational capacity to the group invariance that facilitates symbolic solutions. 
This framework charts a principled foundation for understanding and designing neurosymbolic systems that integrate continuous learning with discrete algebraic reasoning.
\end{abstract}

\input{tex/01_intro}

\input{tex/02_prelim}

\input{tex/03_measure_algebra}
\input{tex/04_phase_transition}

\input{tex/05_practical_implication}

\input{tex/06_conclusion}

\section*{Acknowledgement}
Z. Wang is supported by DAPRA ANSR (RTX CW2231110), 
DARPA TIAMAT (HR0011-24-9-0431), and ARL StAmant (W911NF-23-S-0001). We sincerely appreciate the insightful discussions from Elisenda Grigsby, Yuandong Tian, Kathryn Lindsey, Boris Hanin, and Alvaro Velasquez.

\bibliographystyle{unsrtnat}
\bibliography{ref}

\clearpage
\appendix
\input{tex/xx_appendix}

\end{document}

%% file: tex/01_intro.tex
\section{Introduction}

The integration of neural and symbolic reasoning is a key challenge in advancing the capabilities of modern AI systems. Neural-symbolic AI \citep{chaudhuri2021neurosymbolic,garcez2023neurosymbolic} aims to combine the representational flexibility and approximation power of neural networks with the precision and compositional rigor of symbolic reasoning. Neural networks excel at learning smooth manifolds in high-dimensional parameter spaces and adapting their behavior from large-scale data. Symbolic reasoning, on the other hand, enables exact inference over discrete logical structures and algebraic constraints. Bridging these strengths promises systems that can handle both statistical and combinatorial aspects of complex tasks, leading to improved generalization, alleviated data hunger, and more transparent reasoning processes.

However, existing neural architectures often struggle to internalize true symbolic capabilities, instead relying on \textit{statistical pattern matching} that \textbf{fails when generalizing} beyond the training distribution \citep{zhang2023paradox,valmeekam2023planning}. This underscores the need for theoretical frameworks that explain \emph{how} symbolic structures can emerge from continuous neural training dynamics. Understanding this emergence at a fundamental level can guide architectural choices and training strategies, ultimately shaping the design of robust neurosymbolic systems.

In this work, we propose a novel theoretical framework that reveals how discrete, symbolic reasoning constraints arise naturally from the continuous evolution of neural network parameters under gradient-based optimization.
By modeling neural training as a gradient flow in measure space, we show that:
\begin{itemize}
\item The training loss for reasoning tasks can be reformulated as a composition of potential functions -- termed \textit{monomial potentials} -- whose values are expectations of monomials under the parameter measure. Finding global optimizers then amounts to identifying binary assignments to these potentials that satisfy a prescribed ``logical expression''.
The resulting solution space defined over measures is endowed with ring-like algebraic structure, and the monomial potentials act as ring homomorphisms that preserve this structure. Consequently, general solutions can be algebraically composed from special ones, revealing a high-level compositionality in the landscape of neural network global minimizers.
\item As the neural network trains under geometric constraints, the parameter measure $\mu_t$ exhibits a dual emergence of \textit{gradient decoupling} and \textit{dimension reduction}.
When the velocity field preserves $O(d)$-invariance, the gradient dynamics decouple: the evolution of each monomial potential follows an independent trajectory. This decoupling reduces the infinite-dimensional Wasserstein gradient flow to a coordinate-wise optimization over potential variables.
As a result, training drives each monomial potential toward its optimal (binary) assignment, solving the reasoning task as if a boolean variable satisfaction problem, while ensuring convergence to minimizers that preserve the ring-like algebraic structure.
Along the trajectory toward symbolic solutions, the gradient flow over the measure space progressively stabilizes the dynamics by inducing invariant directions within the Wasserstein tangent space.
This emergence of invariance effectively contracts the system's degree of freedom.
\end{itemize}
Critically, our theory links these geometric and algebraic insights to concrete, actionable principles.
Our theory highlights the value of geometric invariance and informs minimum data complexity for networks to achieve stable solutions equipped with ring-like symbolic operations.
We further discuss how this analysis can guide practical architecture realization.
These insights provide a principled explanation for empirical observations and guide the design of architectures that excel at both continuous representation learning and discrete symbolic reasoning.

We note that our results should be viewed as an idealized foundation rather than a fully comprehensive model. Yet, this work provides the first theoretically grounded account of how continuous neural optimization can yield discrete, symbolic reasoning structures. Our results open new avenues for building neurosymbolic systems that leverage both the flexibility of neural representations and the rigor of symbolic reasoning.

%% file: tex/02_prelim.tex
\section{A Theoretical Framework for Learning to Reason with Neural Networks}\label{sec:prelim}

In this section, we begin by exemplifying a reasoning task following \citet{tian2024composing}.
We then abstract and generalize this family of problems with algebraic, probabilistic, and geometric tools. 

\paragraph{Scope of Reasoning Tasks.}
We use the term \textit{reasoning task} to refer to learning problems whose underlying structure involves \textit{compositional, algebraic, or logical relationships between inputs}, rather than merely statistical pattern recognition. Examples include learning group operations, ring structures, combinatorial rules, or logical inferences that exhibit symbolic regularities. A reasoning task typically demands that the model capture structured transformations (e.g., associativity, commutativity, closure under operations) rather than simply fitting arbitrary mappings.

\subsection{Motivating Example: Abelian Group Reasoning with Neural Networks}\label{sec:abelian_reasoning}

We start with a motivating example on learning an Abelian group operation -- a simple yet canonical case of symbolic composition.
In fact, Abelian group addition has been used to study the reasoning capability of modern neural models \citep{power2022grokking, gromov2023grokking, nanda2023progress}.
Yet we note that the theoretical framework we are to develop will be significantly broader: it applies to any task where the loss depends algebraically on network parameters (e.g., via monomials).
In Sec. \ref{sec:measure_MP} and beyond, we will abstract these features to study the general dynamics by which neural training can uncover a structured and low-degree-of-freedom measure space that exhibits algebraic properties, support symbolic reasoning.

\paragraph{Task Definition.}
Suppose we have a finite Abelian group $(A, \cdot)$ with commutative operation $\cdot$ and cardinality $n = |A|$.
We aim to train a neural network for predicting the output of $a_1 \cdot a_2$ for two group elements $a_1, a_2 \in A$.
The network input consists of the one-hot embeddings of these group elements $e_{a_1}, e_{a_2} \in \real^n$.
The output aims to predict $a_1 \cdot a_2$, also represented in one-hot encoding.

\paragraph{Neural Architecture.} We analyze a two-layer neural network with $q$ hidden nodes and quadratic activations $\sigma(x) = x^2$. Note that quadratic activation function has been a common choice for theoretical analysis \citep{du2018power, zhang2024relu}, and its variants have been used in practice \citep{shazeer2020glu, so2021searching, zhang2024relu}. 
\begin{align}
o(a_1, a_2) = \frac{1}{q} \sum_{j=1}^q w_{cj} \sigma \left( w_{aj}^\top e_{a_1} + w_{bj}^\top e_{a_2} \right),
\end{align}
where weight matrices $W_a, W_b, W_c \in \mathbb{R}^{n \times q}$ encode inputs $e_{a_1}, e_{a_2}$ as hidden features and decode them to predictions, respectively.
Different from the conventional formulation of neural networks, we normalize the final output by $1/q$. This change aligns with the mean-field perspective of neural networks \citep{mei2019mean}.

Following \citet{tian2024composing}, we consider representing and learning the weight matrices in their Fourier space $w_{aj} = \sum_{k \neq 0} z_{akj} F_k, w_{bj} = \sum_{k \neq 0} z_{bkj} F_k, w_{cj} = \sum_{k \neq 0} z_{ckj} \overline{F_k}, \forall j \in [q]$,
where $F_k = [F_k(g)]_{g \in G} \in \complex^n$ are the scaled Fourier basis functions ($0 \le k < n$), and $z_{ak}, z_{bk}, z_{ck} \in \complex$ are the Fourier coefficients.
We further collectively organize these coefficients as matrices $z_j = [z_{akj}, z_{bkj}, z_{ckj}]_{0 \le k < n} \in \complex^{3 \times n}$ for $j \in [q]$.

\paragraph{Loss Function.} We optimize $\{z_j\}_{j \in [q]}$ by minimizing the $L_2$ loss between the predicted and ground-truth group compositions across all possible pairs of group elements:
\begin{align} \label{eqn:loss_nn}
H(\{z_j\}_{j \in [q]}) = \sum_{a_1, a_2 \in A} \left\lVert P^\perp \left( \frac{1}{2n} o(a_1, a_2) - e_{a_1 \cdot a_2} \right) \right\rVert_2^2,
\end{align}
where $P^\perp = I - \frac{1}{n} 1 1^\top$ is the zero-mean projection operator.

We conclude this section by presenting the following proposition showing that the loss function can be reformulated as a combination of a special family of polynomials:
\begin{proposition} \label{prop:loss_decomp}
The loss function $H$ in Eq. \ref{eqn:loss_nn} can be reformulated as: $H = \frac{1}{n-1} \sum_{k \neq 0} \ell_k + \frac{n-1}{n}$,
\begin{align*}
&\ell_k = -2 \rho_{kkk} + \sum_{k_1, k_2} |\rho_{k_1k_2k}|^2 + \frac{1}{4} \left| \sum_{p \in \{a, b\}} \sum_{k'} \rho_{p, k', -k', k} \right|^2 + \frac{1}{4}\sum_{m \ne 0} \sum_{p \in \{a, b\}} \left| \sum_{k'} \rho_{p,k',m-k',k} \right|^2, \\
& \rho_{k_1k_2k} = \frac{1}{q} \sum_j z_{ak_1j} z_{bk_2j} z_{ckj}, \quad \rho_{pk_1k_2k} = \frac{1}{q} \sum_j z_{pk_1j} z_{pk_2j} z_{ckj}.
\end{align*}
\end{proposition}
The proof can be found in the Appendix \ref{proof}. Similar for \textbf{all proofs} hereinafter.

The Abelian-group toy network already exhibits two essential structural ingredients that drive the rest of the paper:
(1) its loss depends on the parameters collectively through their empirical measures; and
(2) this dependence is through empirical averages of some \textit{monic monomials} evaluated against such measures.
Once these observations are isolated, nothing in the analysis hinges on the specific two-layer architecture or on Abelian groups alone. 

In what follows, we therefore ``lift" the finite list of neurons to a probability measure over a smooth parameter manifold and study the resulting gradient flow directly in measure space, where algebraic and geometric properties can be elucidated.

\subsection{From Finite Neurons to Distributional Neuron Space}
\label{sec:measure_MP}

We now make that lifting explicit, and introduce some key mathematical devices that generalize the objective and optimization for the reasoning task shown in Sec. \ref{sec:abelian_reasoning}.
Let $d$ be the ambient parameter space ($d = 3n$ for the example in Sec. \ref{sec:abelian_reasoning}) and $M \subset \real^{d}$ represent the parameter space of the neural network.
Let $P(M)$ denote the space of probability measures on $M$ endowed with the $W_2$-Wasserstein metric and necessary regularity (e.g., finite second moment) \citep{villani2009optimal}, and $P_*(M)$ denote the entire space of ($\sigma$-finite, non-negative, and countably additive) measures over Borel $\sigma$-algebra of $M$.

An crucial observation made by Proposition \ref{prop:loss_decomp} for the Abelian group reasoning task is: the loss $H$ depends on parameters $\{z_j\}_{j \in [q]}$ only through the empirical distribution $\mu^{(q)} = \frac{1}{q}\,\sum_{j=1}^q \delta_{z_j}$ and each $\rho_{\Set{I}}$ in Proposition \ref{prop:loss_decomp} is essentially an expected monomial evaluated against $\mu^{(q)}$: $\rho_{\Set{I}} = \E_{z \sim \mu^{(q)}} [\prod_{i \in \Set{I}} z_i]$, where $\Set{I}$ denotes a set of indices involved in the monomial.
As $q \rightarrow \infty$, $\mu^{(q)}$ converges to a limiting measure $\mu$ in distribution.
This suggests generalizing $H$ to a functional $H[\mu]$ defined for all $\mu \in P(M)$ and tracking the \textit{expected} values of monomials:
\begin{definition}[Monomial Potentials, generalized from \citet{tian2024composing}]\label{def:MP}
Given a set of indices $\Set{I}$, an associated (monic) monomial is defined as $r(z) = \prod_{i \in \Set{I}} z_i$.
We also denote $\Set{I}(r)$ as the indices of variables with non-zero degree in $r$.
A \emph{monomial potential (MP)} $\rho_r: P_*(M) \rightarrow \real$ is defined as the expectation of the specified $r$ against the input measure $\mu$: $\rho_r[\mu] = \E_{z \sim \mu} [r(z)] = \int r(z) d\mu(z)$.
\end{definition}
Here we consider expectation in a generalized sense, which evaluates the Lebesgue integral of a function over some measure.
Although ``MP'' shares the terminology used in \cite{tian2024composing}, it exclusively refers here to the version defined for measures.
Consider a finite set of monomials: $\Set{R} = \{r_1, \cdots, r_m \}$ with $m = |\Set{R}|$, we define the loss functional $H$ with respect to $\Set{R}$ in a separable form, where the effect of each monomial is isolated and encapsulated within the potentials:
\begin{align}
\label{eqn:H_form}
H[\mu] = L(\rho_{r_1}[\mu], \cdots \rho_{r_m}[\mu]),
\end{align}
where $L: \real^m \rightarrow \real$. We further denote $\rho = [\rho_{r_1}[\mu], \cdots \rho_{r_m}[\mu]]$.

Clearly, Proposition \ref{prop:loss_decomp} reveals the Abelian group addition task is one of the special instantiations of Eq. \ref{eqn:H_form}.
We note that terms like $\rho_{k_1 k_2 k}$, $\rho_{p k_1 k_2 k}$ in Proposition \ref{prop:loss_decomp} are all special cases of MPs (Definition \ref{def:MP}).
And function $L(\rho)$ takes the form specified in Proposition \ref{prop:loss_decomp}.
Such generalization also allows for moving beyond the simple example in Sec. \ref{sec:abelian_reasoning}.

\subsection{Optimization through Wasserstein Gradient Flow}
With above setups, we are interested in studying how $\mu$ evolves under a gradient flow that aims to minimize the functional $H[\mu]$, and how this evolution can induce algebraic properties in these MPs.
Conventionally, finite-neuron neural networks are trained by gradient descent. After being lifted to the measure space, we consider its counterpart for probability measures -- Wasserstein gradient flow -- inspired by mean-field perspectives of neural network training \citep{mei2019mean}.
Specifically, let $\{\mu_t\}_{t \ge 0}$ be the trajectory of the measure under the Wasserstein gradient flow given by $H$, satisfying:
\begin{align}
\label{eqn:wass_pde}
\partial_t \mu_t = \nabla_{z} \cdot \left(\mu_t \nabla_{z} \left(\frac{\delta H}{\delta \mu} [\mu_t] \right)\right).
\end{align}
The term $\nabla\tfrac{\delta H}{\delta \mu} [\mu_t]: \real^d \rightarrow \real^d$ is called the \textit{velocity field}.
The upshot is that for each infinitesimal time interval $\tau \rightarrow 0$, $\mu_t$ moves in the steepest descent direction in $W_2$-space: $\mu_{t + \tau} \approx \argmin_{\mu \in P(M)} \{ H[\mu] + \frac{1}{2 \eta_t \tau} W_2(\mu_t, \mu) \}$ for a step size $\eta_t$.
We refer interested readers to \cite{ambrosio2008gradient, villani2009optimal} for more details.

%% file: tex/03_measure_algebra.tex
\section{Algebraic Structure of Solution Space}\label{sec:algebra}

In this section, we reveal the algebraic structure of global minimizers in the measure space.
We will show that by equipping measures with proper operations, we can establish ring-like structures between measures and show that MPs preserve such properties, thus enabling compositionality.

\subsection{Characterization of Global Optimizer}
\label{sec:global_min}
We start with the motivating example on Abelian group reasoning (see Sec. \ref{sec:abelian_reasoning}).
As shown by Proposition \ref{prop:loss_decomp}, the functional loss turns out to be a combination of a series of monomial potentials (MPs) (Definition \ref{def:MP}).
It has not escaped our notice that there exists a boolean assignment of the MP values such that $H = 0$:
\begin{align}
\label{eqn:abelian_sol}
\rho_{kkk} = \mathbb{1}(k \ne 0), \quad
\rho_{k_1 k_2 k} = 0, \quad
\rho_{p k_1 k_2 k} = 0, \quad \forall p \in \{a, b\}, k_1, k_2, k \in [d].
\end{align}
Any measures that satisfies the boolean equation Eq. \ref{eqn:abelian_sol} is a global optimizer of $H$.
Following \cite{tian2024composing}, we can generalize such characterization of global minimizer by the concept of 0/1-sets.
\begin{definition}[0/1-Set, generalized from \cite{tian2024composing}]
\label{def:zero_one_set}
Suppose a measure $\mu \in P_*(M)$ has 0-set $\Set{R}_0 \subset \Set{R}$ and 1-set $\Set{R}_1 \subset \Set{R}$, (or equivalently 0/1-set $(\Set{R}_0, \Set{R}_1)$), then $\rho(\mu) = 0$ for every $\rho \in \Set{R}_0$ and $\rho(\mu) = 1$ for every $\rho \in \Set{R}_1$. 
\end{definition}
The solutions to Eq. \ref{eqn:abelian_sol} has 0-set $\Set{R}_c \cup \Set{R}_n \cup \Set{R}_*$ where
$\Set{R}_c := \{r_{k_1 k_2 k} | k_1,k_2,k\text{ not all equal}\}$, $\Set{R}_n := \{r_{p,k',-k',k}\}$, $\Set{R}_* = \{r_{p,k',m-k',k} | m\neq 0\}$, and 1-set $\Set{R}_g := \{r_{k k k} | k\neq 0\}$ \citep{tian2024composing}.
\textit{As a result, a complex system of neural weights converts to a logical expression, where MPs play a role similar to boolean variables.}
One may already notice that such solutions are obtained by construction.
Even the simple Abelian group reasoning task may have global minimum other than those we construct.
However, in what follows, we will point out (1) that such a solution space already exhibits well-defined algebraic properties that allow for composition; and (2) gradient-based optimization prefers such solutions under geometric constraints.

\subsection{Semi-Ring Structures of Measures and Potentials}

Analogously to \cite{tian2024composing}, we propose to define the following algebraic operations over measures:
\begin{definition}[Algebra over Measures]
\label{def:measure_alg}
For two measures $\mu_1$ and $\mu_2$, define (1) addition as: $\mu_{+} = \mu_{1} + \mu_{2}$ such that $\mu_{+}(A) = \mu_1(A) + \mu_2(A)$ for every measurable $A \subset M$,  (2) multiplication as: $\mu_{*} = \mu_1 * \mu_2$ such that $\mu_{*}$ is the measure of $z_{*} = z_1 \odot z_2$ where $z_1 \sim \mu_1$, $z_2 \sim \mu_2$, $\odot$ denotes element-wise multiplication, and (3) the identity element as $\delta_{\Mat{1}_d}$, i.e., the point mass at the $d$-dimensional all-one vector, and (4) the zero element as the zero measure.
\end{definition}
To make the additive operation well defined, we enlarge the domain to include finite (possibly non-probability) measures.
When an additive structure is not necessary, we may instead restrict our attention to probability measures.
We note that measure-space algebra is entirely different from finite-dimensional vector space (e.g., those defined in \cite{tian2024composing}), where operations such as concatenation and Kronecker product do not explicitly exist.
In our Definition \ref{def:measure_alg}, the addition fuses the mass of two measures, while multiplication describes a process to couple two distributions.

Next, we prove that such algebra endows $P_*(M)$ with a semi-ring structure and MPs function with ring-preserving properties.
\begin{theorem}
\label{thm:ring_homo}
$\langle P_*(M), +, * \rangle$ as in Definition \ref{def:measure_alg} is a commutative semi-ring.
Every MP $\rho_r(\mu)$ is a ring homomorphism: (1) $\rho_r(\mu_1 + \mu_2) = \rho_r(\mu_1) + \rho_r(\mu_2)$, and (2) $\rho_r(\mu_1 * \mu_2) = \rho_r(\mu_1) * \rho_r(\mu_2)$.
\end{theorem}
While \cite{tian2024composing} demonstrates such algebra in the finite-neuron space, our result successfully extends it to measures.
Theorem \ref{thm:ring_homo} guarantees the construction of global minimizers by composing existing solutions in the measure space:
\begin{proposition}
\label{prop:compositionality}
If $\mu_1$ has 0/1-sets $(\Set{R}_0, \Set{R}_1)$ and $\mu_2$ has 0/1-sets $(\Set{S}_0, \Set{S}_1)$, then (1) $\mu_1 * \mu_2$ has 0/1-sets $(\Set{R}_0 \cup \Set{S}_0, \Set{R}_1 \cap \Set{S}_1)$; and (1) $\mu_1 + \mu_2$ has 0/1-sets $(\Set{R}_0 \cap \Set{S}_0, (\Set{R}_1 \cap \Set{S}_0) \cup (\Set{R}_0 \cap \Set{S}_1))$.
Moreover, if $\mu_1$ is a global optimizer and $\mu_2$ has 1-set $\Set{R}$, then $\mu_1 * \mu_2$ is a global optimizer.
\end{proposition}
We point out that all other constructions of optimizers through composition in \cite{tian2024composing} should also apply in our measure-based setting.
By Proposition \ref{prop:compositionality}, we conclude the compositionality of the special solutions to general solutions via the ring-like algebraic properties.

%% file: tex/04_phase_transition.tex
\section{Algebraic Emergence from Geometric Constraints and Dimension Reduction}\label{sec:optim}

Although solution space is automatically equipped with the algebraic structures, it is natural to ask whether gradient-based training identifies those global minimizers.
In \cite{tian2024composing}, the authors provide two initial clues to answer this question: (1) gradient introduces descent directions to optimize each MPs; and (2) at the initialization stage, gradient dynamics for each MP is decoupled.
However, it remains open whether the full optimization trajectory can reach the global minimum that satisfies each MP identity (such as those speficified in Eq. \ref{eqn:abelian_sol}).

In this section, we investigate the evolution of measure $\mu_t$ under the gradient flow framework established in Sec. \ref{sec:prelim}.
We provide two complementary perspectives to interpret how algebraic structures emerge.
We show that $\mu_t$ undergoes two phenomenon spontaneously: (1) measure dynamics translates to gradient flow under geometric constraints that directly optimizes the symbolic expressions over MPs; and (2) during the process of finding such solutions, the solution space undergoes a necessary reduction on the degree of freedom following the Wasserstein gradient flow.

\subsection{Geometric-Constrained Gradient Dynamics Finds Algebraic Solution}
\label{sec:grad_find_alg_sol}

In this section, we provide our theoretical results on how gradient descent converges to the solution space endowed with the algebraic structures.
First of all, we consider a trajectory of measure $\{\mu_t\}_{t \ge 0}$ governed by Wasserstein gradient flow in Eq. \ref{eqn:wass_pde}.
We make the following assumptions for the flow and the loss functional.
\begin{assumption}
\label{ass:coord_descent_MP}
Suppose the following is true: (1) $\mu_0 = \gauss(0, I)$ at the initialization;
(2) $\deg r \ge 3$ and is odd for every $r \in \Set{R}$;
and (3) $\nabla \frac{\delta H}{\delta \mu}[\mu_t]$ is $O(d)$-equivariant.
\end{assumption}
Assumption \ref{ass:coord_descent_MP}(1) is satisfied because neural networks are routinely initialized with i.i.d. Gaussian distribution, $\gauss(0, I)$.
Assumption \ref{ass:coord_descent_MP}(2) bounds the degree and controls the parity of MPs, a quantity determined by the network architecture.
In our modulo addition task (Sec. \ref{sec:abelian_reasoning}), every MP has degree 3, so this assumption holds automatically.
Assumption \ref{ass:coord_descent_MP}(3) encodes a geometric requirement: the velocity field of the Wasserstein gradient flow needs to be $O(d)$-equivariant: $\nabla \frac{\delta H}{\delta \mu}[\mu_t](R z) = R \nabla \frac{\delta H}{\delta \mu}[\mu_t](z)$ for every $z \in M$ and $R \in O(d)$.
Although this condition is not satisfied by default, it can be enforced through an appropriate choice of architecture.
Specifically, if the energy functional $H$ is $O(d)$-invariant, its induced velocity field is $O(d)$-equivariant as well.

We now present our main result, which characterizes the effect of the measure-space Wasserstein gradient flow on each MP.
The evolution of each MP follows a decoupled PDE trajectory -- the loss functional descends in the coordinate system of MP.
\begin{theorem}[Decoupled Dynamics under Geometric Constraints]
\label{thm:coord_descent_MP}
Consider a trajectory of measure $\{\mu_t\}_{t \ge 0}$ governed by Wasserstein gradient flow in Eq. \ref{eqn:wass_pde}. Under Assumption \ref{ass:coord_descent_MP}, each monomial potential is optimized coordinate-wisely as:
\begin{align}
\label{eqn:coord_descent_MP}
\partial_t \rho_{r_i}[\mu_t] = - C_i(t) \partial_{\rho_{r_i}} L(\rho),
\end{align}
where $C_i(t) > 0$ is a time-dependent scalar function only dependent on $\rho_{r_i}$.
\end{theorem}
By Theorem \ref{thm:coord_descent_MP}, the gradient flow of MPs evolves independently. 
Specifically, the time derivative of the $i$-th MP $\partial_t \rho_{r_i}$ corresponds to the directional derivative of the overall loss along the $i$-th MP $\rho_{r_i}$, scaled by a time-dependent factor that depends only on $\rho_{r_i}$ itself.
As a result, the dynamics of $\rho_{r_i}$ are decoupled from the underlying dynamics of $\mu_t$.
\textit{At a high level, this reveals that optimization in the weight space manifests as the gradient descent over the MP variables $\rho$.}
Each MP independently minimizes the subset of loss terms in which it appears, satisfying the logical expression (via 0/1 set assignments) given by the task.
The global optimizer can then be constructed by algebraically composing the solutions of individual MPs, leveraging the underlying algebraic structure (Sec. \ref{sec:algebra}).

To illustrate how the dynamics drive each MP toward satisfying the 0/1-sets it participates in, we return to the Abelian reasoning task described in Sec. \ref{sec:abelian_reasoning}.
Consider the MP $\rho_{kkk}$ for some $k \in [d]$ in Proposition \ref{prop:loss_decomp}, we can derive that $\partial_{\rho_{kkk}} L \propto \rho_{kkk} - 1$.
Then by Theorem \ref{thm:coord_descent_MP}, we find that $
\frac{\partial}{\partial t} \rho_{kkk}[\mu_t] =  C_{kkk}(t) (1 - \rho_{kkk}[\mu_t])$,
which implies the dynamics $\rho_{kkk}[\mu_t] = 1 - \exp(-\overline{C_{kkk}}t)$ (here we consider $\overline{C_{kkk}}$ as a uniform bound of $C_{kkk}(t)$).
This gives an exponential convergence of $\rho_{kkk}$ to the boolean assignment $1$, satisfying the 1-set shown in Eq. \ref{eqn:abelian_sol}.
The same effect applies to other MP variables as well.

This result underscores the value -- if not the necessity -- of incorporating geometric constraints into neural architectures to facilitate the emergence of algebraic structure.
Fundamentally, our derivation of Theorem \ref{thm:coord_descent_MP} leverages the rich interplay between invariant measures and polynomial representations. An $O(d)$-equivariant velocity field preserves the symmetry of the measure, which in turn leads to the cancellation of interactions among MPs that would otherwise entangle and mix the gradients across different components.

In Appendix \ref{sec:prf_devices}, we extend Theorem \ref{thm:coord_descent_MP} to a broader class of settings where the potential is constructed from Hermite polynomials, and the velocity field is required to be anti-symmetric (less restricted than fully $O(d)$-equivariant).

\subsection{Stabilized Measure Dynamics Reveal Reduction on Degree of Freedom}
\label{sec:dim_reduce}

Having shown that a geometrically constrained velocity field projects a measure-space flow onto a simple gradient flow in the MP space, we now analyze the evolution of the underlying measure itself.
Remarkably, the measure dynamics intrinsically pursue parsimony, progressively discarding extraneous degrees of freedom as the flow unfolds.

\paragraph{Low-Entropy Solutions.}
We first characterize a set of measures that meet the MP constraints.
Assume the decoupled MP gradient flow in Theorem \ref{thm:coord_descent_MP} converges to a stationary point $\rho^* \in \real^m$, i.e. $\nabla L(\rho^*) = 0$.
It remains hard to locate a measure in the vast probability space such that monomials evaluated against it match a target value.
The theorem below resolves this difficulty by identifying the measure that simultaneously satisfies the MP constraints and minimizes differential entropy.
\begin{theorem}
\label{thm:low_ent}
For every $\varrho \in \real^m$, consider a measure $\mu$ that realizes $\rho_{r_i}[\mu] = \varrho_i$ for every $i \in [m]$ with minimal differential entropy $E[\mu] := \mean_{\mu}[-\log d\mu]$. Then all such measures form a Riemannian submanifold of $P(M)$, whose dimension is at most $m$.
\end{theorem}
Theorem \ref{thm:low_ent} shows that the entropy-minimizing measures compatible with a set of MP variables live on a finite-dimensional Riemannian submanifold, collapsing the original infinite-dimensional measure space.
This follows from the fundamental result that, under moment constraints, the entropy-minimizing probability measures lie in the exponential family.
The measures realizing MP constraints likely still span an infinite-dimensional space. 
By favoring low-entropy solutions, we project the admissible measures to a low-dimensional manifold with fewer degrees of freedom.
Beyond entropy minimization, divergence-based regularizers yield analogous results (see Remark \ref{rmk:kl_sol}).
Nevertheless, the loss functional does not explicitly penalize entropy.
We borrow concepts from Renormalization Group (RG) theory to clarify why the training dynamics drive a spontaneous reduction in effective degrees of freedom.

\paragraph{A Renormalization Group Perspective.}
In RG analysis \citep{goldenfeld2018lectures}, it is shown that the Hessian's eigenspectrum at a fixed point of a gradient flow encodes its stability.
In particular, the stable manifold, defined as the set of initial states whose trajectories converge to the fixed point, has an intrinsic dimension equal to the dimension of the eigenspaces with negative eigenvalues.
Because every trajectory that survives the long-time limit lies on this lower-dimensional manifold, its geometry fully determines the asymptotic behaviour of the system.
Hence, the appearance of negative Hessian eigenvalues reduces the degrees of freedom of the system.
In this section, we reveal this effect for the measure-space flow by analyzing the eigenspace of the second variation.

We denote the second variation of $H$ at $\mu_t$ as a time-varying operator $\mathbb{L}(t) = \frac{\delta^2 H}{\delta \mu^2}[\mu_t]: T_{\mu_t}P(M) \rightarrow T_{\mu_t}P(M)$.
We define $(\lambda(t), v_t)$ as an eigenpair of $\mathbb{L}(t)$ which satisfies $\mathbb{L}(t)[v_t] = \lambda(t) v_t$, where $\lambda(t)$ is the time-varying eigenvalue and $v_t$ is the eigenfunction.
We endow $\mathbb{L}(t)$ with idealized properties for the ease of analysis: $\lambda(t)$ and $v_t$ are both real-valued and no multiplicity exists when $\lambda(t) \ne 0$.
Below we present a formal statement of degree of freedom reduction:
\begin{theorem}[Degree of Freedom Reduction]\label{thm:crit_times}
Consider loss functional in Eq. \ref{eqn:H_form}, all eigenfunctions corresponding to non-zero eigenvalues of second variation $\mathbb{L}(t)$ lie in a subspace spanned by the monomial set $\Set{R}$, i.e, $v_i \subset \Span(\Set{R})$ for every $\lambda_i \ne 0$.
Moreover, under Assumption \ref{ass:coord_descent_MP}, and suppose $[\nabla^3 L]_i \nabla_i L \succeq 0$ for every $i \in [m]$, $\mathbb{L}(t)$ will have non-increasing eigenvalues $\frac{d}{dt} \lambda(t) \le 0$.
There will be finitely many $0 < t_1 < \cdots < t_{m'}$ ($m' \le m$) at which an eigenvalue of $\mathbb{L}(t)$ crosses zero.
\end{theorem}
Theorem \ref{thm:crit_times} establishes that the eigenspace of the Hessian operator is entirely determined by the prescribed monomials $\Set{R}$.
This implies that the intrinsic degrees of freedom in the system are bounded above by $m$.
This result solidates and generalizes Theorem \ref{thm:low_ent}: under Wasserstein gradient flow, the measure evolves within a substructure with at most $m$ intrinsic degrees of freedom.
Our observation aligns with the decoupled dynamics of MPs (see Sec. \ref{sec:grad_find_alg_sol}), where the network evolution is governed by the optimization gradient flow over MPs, again reflecting $m$ effective dimensions.

If we additionally impose constraints on the gradient and higher-order derivatives of the loss function $L$ (rather than the overall loss functional $H$), the eigenvalues of the induced operator $\mathbb{L}(t)$ are shown to decrease monotonically over time, signifying a progressive contraction of the eigenspace.
When an eigenvalue crosses zero, its associated eigenfunction transitions into a flat direction.
Fluctuating the measure along flat directions incurs controllable second-order cost, indicating that these directions correspond to stable fixed points in the space of probability measures $P(M)$.
According to RG \citep{goldenfeld2018lectures}, the appearance of a new eigenspace with negative eigenvalues produces an asymptotically invariant space, shaving off the system's effective degree of freedom.
$\dim\Span(\Set{R}) \le m$ decides the maximum times dimension reduction can happen.

%% file: tex/05_practical_implication.tex
\section{Embedding Geometric Constraints into Neurosymbolic System}\label{sec:4}

Besides the measure-theoretic framework analyzing training dynamics of a reasoning task, we add insights to practical considerations in designing neurosymbolic AI systems. 

\subsection{Sample Complexity for Learning Symbolic Reasoning}
\label{sec:sample_complex}

Sec. \ref{sec:grad_find_alg_sol} shows that enforcing an $O(d)$-equivariant velocity field is sufficient to decouple the gradient dynamics and directly optimize for MP constraints.
Extending this insight, we now consider reasoning tasks that require a general $G$-equivariant velocity field.
This can be achieved by training a $G$-invariant neural network. 
We then establish learnability guarantees for such problems.

\begin{theorem}[Informal, Sample Compelxity for Learning Invariants]\label{thm:sample_complex}
Suppose $G$ is a Lie group and $M_d$ is a data manifold. Consider a family of $G$-invariant functions $\Set{F}^s(M_d)$, square-integrable up to order $s > 0$ over $M_d$.
Denote $d' = \dim(M_d / G)$ and let $s = (1 + \kappa) d'/2$ for some positive integer $\kappa \ge 0$. Given $\theta \in (0, 1]$, and a $G$-invariant function $f^* \in \Set{F}^{\theta s}(M_d)$, then with probability at least $1-\delta$, empirical risk minimization can learn $\epsilon$-approximate $G$-invariant function $\hat{f}$ with $n$ many samples, where $n = \Theta\left(\max\left\{ 1/(|G|\epsilon^{1+1/\theta(\kappa+1)}), \log(1/\delta)/\epsilon^2\right\}\right)$ for finite $G$
and $n = \Theta\left(\max\left\{ \vol(M_d / G)/\epsilon^{1+1/\theta(\kappa+1)} , \log(1/\delta)/\epsilon^2\right\}\right)$ for infinite $G$.
\end{theorem}
Theorem \ref{thm:sample_complex} is based on seminal results presented in \citet{tahmasebi2023exact}.
We defer a formal statement to Appendix \ref{sec:prf_sample_complex}.
Under both finite and infinite scenarios, the PAC sample complexity splits into two components: (1) the expected approximation error, and (2) estimation error.
The approximation term dominates as $\theta \rightarrow 0$, meaning a less smooth $f^*$ results in a harder target to approximate.
Conversely, when the hypothesis class is sufficiently smooth ($\kappa \rightarrow \infty$), the estimation term becomes dominant.
For continuous $f^*$, $n = \Theta(1/\epsilon^2)$ needs to scale quadratically with $1/\epsilon$.
When $G$ is finite, the sample complexity is scaled down by the group's cardinality.
While for infinite $G$, the improvement arises through contraction of the effective volume $\vol(M_d / G)$, which happens when the orbit of $G$ covers a large portion of $M_d$.

\subsection{Geometric Constraints on Architecture Design}
\label{sec:4.2}

\paragraph{Preserving Group Equivariance.}
Our Theorem \ref{thm:coord_descent_MP} provides a sufficient condition for reaching a symbolic solution via $O(d)$-equivariant velocity flow.
Hence, choosing loss functions or module designs that preserves group invariance (e.g. \cite{zhu2025rethinking}) is essential for realizing symbolically-structured reasoning.
As discussed in Theorem~\ref{thm:sample_complex}, architectures learning structural symmetries necessitate a minimal number of data samples.
Learning with sufficient data in turn guarantees gradient decoupling and global convergence of MPs under the measure-space flow.

\paragraph{Enforcing Reduction of Effective Dimension.}
As seen in Theorem \ref{thm:low_ent}, entropy regularization or, in practice, Gaussian weight initialization coupled with weight decay can steer the training dynamics toward solutions in a streamlined substructure.
Degree of freedom reduction (Theorem~\ref{thm:crit_times}) is a necessary for obtaining a symbolic solution, assuming a joint constraint involving the first- and third-order derivatives of $L$.
This essentially requires a non-increasing Hessian, likely achieved by enforcing Lipschitz continuity, bounding layer gradients, or using mildly smooth activations.

\paragraph{Ensuring Compatibility for Monomial Potentials.}
Central to our analysis is a loss functional containing emergent components -- MPs (Proposition \ref{prop:loss_decomp}).
Reproducing this algebraic structure to new tasks requires non-trivial alignment between input/output formulations, loss functions, and model architecture with our framework.
Importantly, many algorithmic reasoning tasks (e.g., state tracking) can be reframed as group theory problems (e.g., word problems) \citep{merrill2023parallelism, merrill2024illusion}, potentially inheriting the mathematical properties of our formulation.

%% file: tex/06_conclusion.tex
\section{Conclusion and More Discussions}
\label{sec:conclusion}
We have developed a measure-theoretic and geometric framework that explains how discrete symbolic capabilities can emerge in neural networks through continuous gradient-flow training.
By modeling network optimization as a gradient flow over measure space, we show that two complementary processes, gradient decoupling and degree of freedom reduction, naturally arise under geometric constraints.
These dynamics drive the measure space onto lower-dimensional substructure while simultaneously isolating and optimizing algebraic components encoded as monomial potentials.
Crucially, we demonstrate that these potentials endow the space of measures with a ring-like structure, enabling the compositional assembly of global solutions from simpler algebraic pieces.
This measure-theoretic perspective bridges the gap between continuous optimization and discrete symbolic reasoning, offering rigorous conditions under which neural networks can internalize and generalize structured logical patterns.
We outline \textbf{future research opportunities} in Appendix \ref{sec:future}.

%% file: tex/xx_appendix.tex
\section{Proofs}\label{proof}

\subsection{Remarks on Idealized Assumptions}\label{sec:assume}

Before delving into the formal proofs, we summarize the standing conventions, simplifications, and default assumptions adopted throughout.
\begin{enumerate}
\item A tacit but critical assumption throughout our analysis is that all parameters are real-valued.
The proofs exploit algebraic identities specific to real-valued orthogonal polynomials (e.g., the Hermite family).
Moving to complex numbers will alter those properties.
Although the parameter space defined in \cite{tian2024composing} is formally complex because of the Fourier transform, a real-valued signal yields conjugate-symmetric Fourier coefficients that can be mapped back to a real number (e.g., via the Hartley transform).
For clarity, we therefore restrict attention to the general structure of the problem (MPs, separable losses, and related constructs) and analyze it assuming all values are real.
We leave a rigorous extension to complex parameter spaces as an open direction for future work.
\item Throughout the proofs we also assume the second-variation (Hessian) operator is diagonalizable and self-adjoint with respect to the relevant $L^2$ or Wasserstein inner product (we shorthand $L^2(\mu)$ as $L(\mu)$).
Self-adjointness guarantees a real spectrum and an orthonormal eigenbasis, allowing us to treat eigenvalues and eigenvectors as real quantities.
The assumption streamlines notation, avoiding complex conjugates in derivations, but is not essential.
Every argument extends verbatim to the general case provided the corresponding complex-valued eigenpairs are tracked with consistent conjugate notation.
\item Unless stated otherwise, every vector field in our analysis is taken to be smooth and either compactly supported in the interior of $M$ or to possess a vanishing normal component on $\partial M$.
These conditions guarantee that the boundary term in Gauss’s (divergence) theorem is zero, so integration by part carries no surface contribution.
\end{enumerate}

\input{tex/xx_prf_alg_struct}
\input{tex/xx_prf_devices}

\input{tex/xx_prf_coord_descent}

\input{tex/xx_prf_dim_reduction}
\input{tex/xx_prf_sample_complex}

\section{Future Directions}\label{sec:future}

We discuss several crucial expansions of the present framework, both theoretically (Sec. \ref{D1}, \ref{D2}, and \ref{Dnew}) and practically (Sec. \ref{D3} and \ref{D4}). First, we seek to handle learning tasks with non-Abelian groups (non-solvable groups) and even continuous group reasoning, together with more general polynomial constraints capturing the broader range of algebraic constraints found in real-world tasks.
Next, we aim to formulate and empirically validate neural scaling laws specific to neurosymbolic models, clarifying how group structure and ring compatibility reduce required model capacity. Then, we discuss the role of static v.s. dynamic data assumptions. Finally, we plan to translate our theoretical findings into concrete design principles for practical architectures to tackle the growing complexity of modern neurosymbolic applications.

\subsection{Beyond Abelian Group and MP-Based Settings}\label{D1}
Our framework assumes the task group $A$ to be a finite Abelian group, where the group operation is commutative.
These structural properties might be critical in enabling a tractable decomposition of group reasoning tasks via MPs.
The group word problem has been widely used to characterize the algorithmic reasoning capabilities of modern neural architectures, such as transformers \citep{merrill2023parallelism, merrill2024illusion}.
In particular, architectures capable of solving the word problem equipped with non-solvable groups are known to possess $\mathsf{NC^1}$-complete circuit representation power.
As such, analyzing non-Abelian groups remains a non-trivial and technically challenging direction.
Moreover, infinite groups are ubiquitous in real-world applications.
For example, modeling compositions of rotation groups plays a fundamental role in vision tasks.
Establishing theoretical guarantees for training neural architectures over such groups is crucial for advancing the design of vision foundation models and other multi-modal AI systems.
While our theory extends to the Hermite polynomial family, it is plausible that a more general family of polynomials can persist after decomposition for some other reasoning tasks.
However, these cases require a more careful and nuanced treatment that lies beyond the scope of our current framework.

\subsection{Beyond Single-Task Parameter Space}\label{D2}

Our theory in Sec. \ref{sec:algebra} reveals a promising algebraic structure underlying reasoning tasks.
We define weight-space operators \citep{navon2023equivariant, xu2022signal} with a ring homomorphism structure in Definition \ref{def:measure_alg}, allowing MP constraints to be composed both analytically and algebraically.
A natural yet challenging extension is to consider multi-task or compositional reasoning scenarios, where the goal is to solve each subtask independently and then compose the corresponding optimal weights using the defined weight-space operations.
For example, we may wish to equip an LLM with the ability to perform both code generation and mathematical reasoning by training it separately on code and mathematical corpora, and subsequently combining the resulting models into a unified one with both capabilities, e.g., \citep{zhao2024model}.
This requires identifying explicit structural relationships among subtasks in the MP space and translating them into corresponding algebraic expressions defined on the weight space.
Another intriguing future direction is to leverage our framework to inverse symbolic structures from learned weights.
Specifically, given the weights for one subtask and the complete weights for the final composed task, we may aim to recover the remaining unseen subtask weights by formulating the problem as a weight-space factorization under our proposed algebraic operations.

\subsection{Static versus Dynamic Data Assumptions}\label{Dnew}

Our analysis models neural network training as a gradient flow minimizing a static loss functional $H[\mu]$, implicitly assuming a stationary data distribution.
In real-world scenarios -- especially online or non-stationary settings -- where the data distribution shifts over time, the assumptions underpinning steady dimension reduction may break down.
For instance, if the data continuously evolves, the minimizer of $H$ itself may drift, potentially preventing stable reduction on the degree of freedom.
Related phenomena have been observed in iterative training of generative models (e.g., self-distillation or student-teacher setups), where repeated training on self-generated data can cause model to collapse onto a degenerate subspace \citep{alemohammadself}. 

While both settings involve a kind of parameter concentration, our theoretical framework differs in that it assumes a fixed, external objective, rather than a self-referential or shifting target distribution.
Extending the measure-theoretic analysis to dynamic or online data settings -- where the target distribution co-evolves with the model -- would likely require analyzing non-autonomous Wasserstein gradient flows or time-varying functionals $H_t[\mu]$, an interesting direction for future work.

\subsection{Exploring Neural Scaling Laws for Neurosymbolic Models}\label{D3}
As shown in Sec. \ref{sec:sample_complex}, we establish a theoretical sample complexity bound that guarantees the learnability of group-invariant architectures for symbolic reasoning through realizing an equivariant gradient flow. 
This result can be interpreted as a data scaling law for training such neurosymbolic models
A natural extension is to derive corresponding scaling laws with respect to the number of model parameters.
Our framework reduces the direct analysis of neurosymbolic models to that of geometric neural networks.
Consequently, parameter scaling laws can be studied through the minimal network width required to represent data objects with prescribed group invariance, following approaches as in \cite{wang2023polynomial, dym2024low}.

This perspective also motivates the pursuit of formal scaling laws for neurosymbolic models, analogous to how large-scale language models demonstrate emergent capabilities upon crossing critical thresholds in data or parameter counts.
Rather than merely asserting that the model width must be large, such laws aim to predict how composite structural properties, such as symbolic constraints (e.g., group symmetries) and data complexity (e.g., input length, vocabulary size), determine the necessary architecture size.

We conjecture that neurosymbolic models that explicitly encode the relevant symmetry or ring-compatible priors may exhibit provably better scaling behavior than conventional models that must learn these structures implicitly from data. Empirical investigations could systematically vary the complexity of group reasoning tasks, testing whether our predicted dimensionality and complexity bounds correspond to the onset of emergent symbolic reasoning capabilities.

\subsection{Guidelines for Designing More Practical Neurosymbolic Architectures}\label{D4}

Our theory underscores the pivotal role of group invariance -- particularly $O(d)$-invariance -- in facilitating learning within neurosymbolic systems (Theorem \ref{thm:coord_descent_MP}).
Encoding the appropriate symmetry into the model is not merely a design choice but a fundamental requirement for capturing the algebraic and geometric structures inherent in symbolic reasoning tasks.
In this context, neural architectures endowed with the desired group invariance can serve as a foundation for constructing new families of models that are better aligned with the structural demands of neurosymbolic reasoning.
Beyond incorporating such invariance into models during initial design, it is also essential to explore post-hoc symmetry alignment for pre-trained models. 
That is, given a large pre-trained model, one may seek to adapt its parameters to encode group-invariant properties, effectively ``baking in'' the desired symmetry after pre-training.
This strategy can be implemented by introducing additional architectural blocks with carefully designed initializations that preserve the behavior of the original model at the start point, while enabling fine-tuning to gradually incorporate the desired symmetries (e.g., \cite{zhu2025rethinking}).

Another promising direction is to design architectures that explicitly favor low-entropy solutions in the weight space.
Such low-entropy solutions may correspond to weights that are sparse, low-rank, or otherwise concentrated in structured subspaces (Theorem \ref{thm:low_ent}).
This perspective connects the notion of parsimony and simplicity in weight space to symbolic interpretability in function space.

By systematically applying these heuristics across varied tasks -- from discrete logic puzzles to continuous transformations in robotics -- one can concretely test the measure-based gradient decoupling and degree of freedom reduction hypothesis in real neurosymbolic models.
The hope is to show that explicitly embedding $G$-invariance and ring-compatibility, along with the appropriate scaling laws for network size, enables robust and efficient neural reasoning, bridging symbolic and sub-symbolic paradigms.

%% file: tex/xx_prf_alg_struct.tex
\subsection{Proofs for Algebraic Structures of Solution Space} \label{sec:prf_alg_struct}

\begin{proposition}[Restatement of Proposition \ref{prop:loss_decomp}]
The loss function $H$ in Eq. \ref{eqn:loss_nn} can be reformulated as: $H = \frac{1}{n-1} \sum_{k \neq 0} \ell_k + \frac{n-1}{n}$,
\begin{align*}
&\ell_k = -2 \rho_{kkk} + \sum_{k_1, k_2} |\rho_{k_1k_2k}|^2 + \frac{1}{4} \left| \sum_{p \in \{a, b\}} \sum_{k'} \rho_{p, k', -k', k} \right|^2 + \frac{1}{4}\sum_{m \ne 0} \sum_{p \in \{a, b\}} \left| \sum_{k'} \rho_{p,k',m-k',k} \right|^2, \\
& \rho_{k_1k_2k} = \frac{1}{q} \sum_j z_{ak_1j} z_{bk_2j} z_{ckj}, \quad \rho_{pk_1k_2k} = \frac{1}{q} \sum_j z_{pk_1j} z_{pk_2j} z_{ckj}.
\end{align*}
\end{proposition}
\begin{proof}
Let $\tilde{\Mat{w}}_{cj} = \frac{1}{q} \Mat{w}_{cj}$, and define $\tilde{z}_{cjk}$ accordingly. Then by Theorem 1 of \cite{tian2024composing}, the $\ell_k$ in the loss decomposition can be written as:
\begin{align*}
\ell_k = -2 \tilde{\rho}_{kkk} + \sum_{k_1, k_2} |\tilde{\rho}_{k_1k_2k}|^2 + \frac{1}{4} \left| \sum_{p \in \{a, b\}} \sum_{k'} \tilde{\rho}_{p, k', -k', k} \right|^2 + \frac{1}{4}\sum_{m \ne 0} \sum_{p \in \{a, b\}} \left| \sum_{k'} \tilde{\rho}_{p,k',m-k',k} \right|^2
\end{align*}
where 
\begin{align*}
\tilde{\rho}_{k_1k_2k} = \sum_j z_{ak_1j} z_{bk_2j} \tilde{z}_{ckj}, \quad \tilde{\rho}_{pk_1k_2k} = \sum_j z_{pk_1j} z_{pk_2j} \tilde{z}_{ckj}.
\end{align*}
We conclude the proof by noticing that $\tilde{z}_{ckj} = \frac{1}{q} z_{ckj}$ due to the linearity of Fourier transform.
\end{proof}

\begin{theorem}
[Restatement of Theorem \ref{thm:ring_homo}]
$\langle P_*(M), +, * \rangle$ as in Definition \ref{def:measure_alg} is a commutative semi-ring.
Every MP $\rho_r[\mu]$ is a ring homomorphism: (1) $\rho_r[\mu_1 + \mu_2] = \rho_r[\mu_1] + \rho_r[\mu_2]$, and (2) $\rho_r[\mu_1 * \mu_2] = \rho_r[\mu_1] * \rho_r[\mu_2]$.
\end{theorem}
\begin{proof}
The proof for the first statement can be done by verifying the properties required for a commutative semi-ring.
For of all, let us examine the addition property.
\begin{enumerate}
\item (Associativity) $(\mu_1(A) + \mu_2(A)) + \mu_3(A) = \mu_1(A) + (\mu_2(A) + \mu_3(A))$ for every $\mu_1, \mu_2, \mu_3 \in P_*(M)$ and measurable $A \subset M$.
\item (Commutativity) $\mu_1(A) + \mu_2(A) = \mu_2(A) + \mu_1(A)$ for every $\mu_1, \mu_2 \in P_*(M)$ and every measurable $A \subset M$.
\item (Addition identity) Plus zero measure to any measure is identical to the measure.
\end{enumerate}
For of all, let us examine the multiplication property.
\begin{enumerate}
\item (Associativity) The product over measures $\mu_1 \otimes \mu_2 \otimes \mu_3$ is associative and the product on the sample space is also associative.
\item (Commutativity) Similarly, the product over measures $\mu_1 \otimes \mu_2 \otimes \mu_3$ is commutative and the element-wise product on the samples are also commutative.
\item (Multiplication identity) Consider $z \sim \mu$ for some $\mu \in P_*(M)$, $z \odot \Mat{1}_{d} = z$.
\end{enumerate}
The semi-ring is not equipped with an inverse element for addition.
The second statement can be shown by checking: for every $r \in \Set{R}$ and $\mu_1, \mu_2 \in P_*(M)$:
\begin{align*}
\rho_r[\mu_1 + \mu_2] &= \int r(z) d(\mu_1(z) + \mu_2(z)) \\ &= \int r(z) d\mu_1(z) + \int r(z) d \mu_2(z) = \rho_r[\mu_1] + \rho_r[\mu_2],
\end{align*}
which confirms the addition preserving structure, and:
\begin{align*}
\rho_r[\mu_1 * \mu_2] &= \E_{z \sim \mu_1 * \mu_2} [r(z)] = \E_{z_1 \sim \mu_1, z_2 \sim \mu_2}[r(z_1 \odot z_2)] \\
&= \E_{z_1 \sim \mu_1, z_2 \sim \mu_2}\left[\prod_{k \in \Set{I}} z_{1,k} z_{2,k}\right]
= \E_{z_1 \sim \mu_1, z_2 \sim \mu_2}\left[\prod_{k \in \Set{I}} z_{1,k} \prod_{k \in \Set{I}} z_{2,k}\right] \\
&= \E_{z_1 \sim \mu_1}\left[\prod_{k \in \Set{I}} z_{1,k} \right] \E_{z_2 \sim \mu_2}\left[\prod_{k \in \Set{I}} z_{2,k}\right],
\end{align*}
which confirms the preservation of multiplication operator.
Obviously, $\rho_r(\Mat{1}_d) = 1$.
\end{proof}

\begin{proposition}
[Restatement of Proposition \ref{prop:compositionality}]
If $z_1$ has 0/1-sets $(\Set{R}_0, \Set{R}_1)$ and $z_2$ has 0/1-sets $(\Set{S}_0, \Set{S}_1)$, then (1) $z_1 * z_2$ has 0/1-sets $(\Set{R}_0 \cup \Set{S}_0, \Set{R}_1 \cap \Set{S}_1)$; and (1) $z_1 + z_2$ has 0/1-sets $(\Set{R}_0 \cap \Set{S}_0, (\Set{R}_1 \cap \Set{S}_0) \cup (\Set{R}_0 \cap \Set{S}_1))$.
Moreover, if $z_1$ is a global optimizer and $z_2$ has 1-set $R$ (the entire set of MPs), then $z_1 * z_2$ is a global optimizer.
\end{proposition}
\begin{proof}
Proof is straightforward by verification using Theorem \ref{thm:ring_homo}.
For every $r \in \Set{R}_0 \cup \Set{S}_0$, either $r(z_1) = 0$ or $r(z_2) = 0$, then we have $r(z_1 * z_2) = r(z_1) * r(z_2) = 0$.
For every $r \in \Set{R}_1 \cap \Set{S}_1)$, we have $r(z_1 * z_2) = r(z_1) * r(z_2) = 1$.
For every $r \in (\Set{R}_1 \cap \Set{S}_0) \cup (\Set{R}_0 \cap \Set{S}_1)$, we have either $r(z_1) = 0, r(z_2) = 1$ or $r(z_0) = 0, r(z_1) = 1$, then $r(z_1 + z_2) = r(z_1) + r(z_2) = 1$.
For every $r \in \Set{R}_0 \cap \Set{S}_0$, we have $r(z_1 + z_2) = r(z_1) + r(z_2) = 0$.
For every $r \in (\Set{R}_1 \cap \Set{S}_0) \cup (\Set{R}_0 \cap \Set{S}_1)$, we have either $r(z_1) = 0, r(z_2) = 1$ or $r(z_0) = 0, r(z_1) = 1$, then $r(z_1 + z_2) = r(z_1) + r(z_2) = 1$.
If $z_2$ has 1-set $R$, then $\forall r \in \Set{R}$, $r(z_2) = 1$. This means $r(z_1 * z_2) = r(z_1) * r(z_2) = r(z_1)$.
Since $\{r_1(z_1), \cdots, r_m(z_1)\}$ is a global optimizer, then $z_1 * z_2$ has the same value for all $r \in R$, thus $z_1 * z_2$ is a global optimizer as well.
\end{proof}

\begin{remark}
A more rigorous way to define ``multiplication'' between two measures is through pushforward or the change-of-variable formula. The density function $w$ of $\omega = \mu * \nu$ can be written as: $w(z) = \int u(x) v(z / x) \prod_{i=1}^{d} \frac{1}{|x_i|} dx$, where the division is elementwise, the integral is taken over the $(\real \setminus \{0\})^d$, $u, v$ are density functions of $\mu$ and $\nu$.
\end{remark}

%% file: tex/xx_prf_devices.tex
\subsection{Some Preliminaries and Auxiliary Results}
\label{sec:prf_devices}

In this section, we introduce some common definitions and results useful for the remainder of the proof.
First of all, we will generalize monic monomials to Hermite polynomials.

\begin{definition}
\label{def:herm_poly}
We denote $h_{\alpha}: \real^d \rightarrow \real$ as Hermite polynomial with multi-index $\alpha \in \integer^d$ over $d$ variables.
Consider a finite family of Hermite polynomials $\Set{H} = \{ r_1, \cdots, r_m \}$ and their associative multi-indices $\{ \alpha_1, \cdots, \alpha_m \}$, i.e., $r_i = h_{\alpha_i}$ for every $i = 1, \cdots, d$.
\end{definition}

\begin{definition}[Hermite potential]
\label{def:hp}
Given a measure $\mu \in P(\real^d)$ and a Hermite polynomial (HP) $h_{\alpha}$, we define Hermite potential as $\rho_{h_{\alpha}}(\mu) = \E_{z\sim \mu}[h_{\alpha}(z)] = \int h_{\alpha}(z) d\mu(z)$.
\end{definition}

\noindent Since all MPs we consider are HP, then all results below apply to MP cases.
\begin{lemma}
\label{MP_is_HP}
All MPs are HPs.
\end{lemma}
\begin{proof}
Observe that $r(z) = \prod_{k \in \Set{I}} z_{k}$ is a monomial where no variable is squared or raised to power larger than 2.
Therefore, $r(z) = h_{\alpha}$ where $\alpha_{k} = 1$ if and only if $k \in \Set{I}$, otherwise $\alpha_{k} = 0$.
Since all such monomials are special cases of Hermite polynomials, MPs are special cases of HPs.
\end{proof}

\noindent We define invariance and equivariance with respect to a group as below:
\begin{definition}(Invariance and Equivariance)
Suppose $G$ is a group acting on a domain $\Set{Z}$.
A function $f: \Set{Z} \rightarrow \real$ is called $G$-invariant if $f(g \cdot z) = f(z)$ for every $g \in G$ and $z \in \Set{Z}$.
A function $f: \Set{Z} \rightarrow \Set{Z}$ is called $G$-equivariant if $f(g \cdot z) = g \cdot f(z)$ for every $g \in G$ and $z \in \Set{Z}$.
\end{definition}

\noindent We also formalize three major assumptions throughout the paper.
\begin{assumption}[Initialization]
\label{ass:app:init_mu}
Suppose $\mu_0 = \gauss(0, I)$.
\end{assumption}

\begin{assumption}[Degree constraints]
\label{ass:app:degree}
Suppose for every $r \in \Set{H}$, $\deg r \ge 3$ and is odd.
\end{assumption}

\begin{assumption}[Geometric constraints on gradient flow]
\label{ass:app:fo_var}
Assume the velocity field $\nabla \frac{\delta H}{\delta \mu}[\mu_t]$ is an odd function , i.e., $\nabla \frac{\delta H}{\delta \mu}[\mu_t](-z) = - \nabla \frac{\delta H}{\delta \mu}[\mu_t](z)$ for every $z \in M$.
\end{assumption}

\noindent Below are the lemmas common for the main results:

\begin{lemma}[Velocity field]
\label{lem:velocity}
The velocity field of $H$ at arbitrary measure $\mu \in P(M)$ can be simplified as:
\begin{align*}
\nabla \frac{\delta H}{\delta \mu}[\mu](z) = \sum_{j=1}^{m} \partial_{\rho_{r_j}} L(\rho_t) \nabla r_j(z).
\end{align*}
\end{lemma}
\begin{proof}
We obtain the result via chain rule.
First of all, we derive the first variation of $H$:
\begin{align*}
\frac{\delta H}{\delta \mu}[\mu] = \sum_{j=1}^{m}\partial_{\rho_{r_j}} L(\rho_t) \frac{\delta}{\delta \mu}\left[ \int r_j(z) d\mu(z) \right] = \sum_{j=1}^{m}\partial_{\rho_{r_j}} L(\rho_t) r_j.
\end{align*}
Then we can obtain the gradient of this first variation:
\begin{align*}
\nabla \frac{\delta H}{\delta \mu}[\mu](z) = \nabla \left[ \sum_{j=1}^{m}\partial_{\rho_{r_j}} L(\rho_t) r_j(z) \right] = \sum_{j=1}^{m}\partial_{\rho_{r_j}} L(\rho_t) \nabla r_j(z),
\end{align*}
as desired.
\end{proof}

\begin{lemma}[Second variation]
\label{lem:second_var}
Consider loss functional in Eq. \ref{eqn:H_form}, we derive its second variation as below:
\begin{align*}
\mathbb{L}_t[f, g] = \sum_{i,j=1}^m \partial^2_{ij} L(\rho_t) \left( \int r_i(z) g(z) d\mu_t(z) \right) \left( \int r_j(z) f(z) d\mu_t(z) \right).
\end{align*}
\end{lemma}
\begin{proof}
By direct derivation, we have:
\begin{align*}
\mathbb{L}_t[f, g] &= \left.\frac{d}{d \epsilon d \eta} H[\mu + \epsilon f + \eta g]\right\vert_{\eta = 0} \\
&= \sum_{j=1}^{m} \left.\frac{d}{d \eta} \int \partial_{\rho_{r_j}} L(\rho_t) r_j(z) d (\mu + \eta g(z)) \right\vert_{\eta = 0} \\
&= \sum_{i,j=1}^m \partial^2_{ij} L(\rho_t) \left( \int r_i(z) g(z) d\mu_t(z) \right) \left( \int r_j(z) f(z) d\mu_t(z) \right),
\end{align*}
as desired.
\end{proof}

\begin{lemma}
\label{lem:inv_flow}
Given an Abelian group $G$, suppose $\mu_0$ is $G$-invariant and $\nabla \frac{\delta H}{\delta \mu}[\mu_t](z)$ is $G$-equivariant, then $\mu_t$ is $G$-invariant for all $t \ge 0$.
\end{lemma}
\begin{proof}
Let $g \in G$ be an arbitrary group action.
Consider a new curve of measure $\{\nu_t\}_{t \ge 0}$ defined as $\nu_t = g_{\#} \mu_t$.
Consider a $C^\infty$ smooth test function $\xi: \real^d \rightarrow \real$:
\begin{align*}
\frac{\partial}{\partial t} \int \xi(z) d \nu_t(z) &\overset{(i)}{=} \frac{\partial}{\partial t} \int \xi(g \cdot z) d \mu_t(z)
\overset{(ii)}{=} \int \xi(g \cdot z) \nabla \cdot \left(\mu_t \nabla \frac{\delta H}{\delta \mu}[\mu_t](z) \right) dz \\
&\overset{(iii)}{=} -\int \nabla \xi(g \cdot z)^\top  \nabla \frac{\delta H}{\delta \mu}[\mu_t](z) d\mu_t(z)\\
&\overset{(iv)}{=} -\int \nabla \xi(g \cdot z)^\top g \cdot \nabla \frac{\delta H}{\delta \mu}[\mu_t](z) d\mu_t(z) \\
&\overset{(v)}{=} -\int \nabla \xi(g \cdot z)^\top \nabla \frac{\delta H}{\delta \mu}[\mu_t](g \cdot z) d\mu_t(z),
\end{align*}
where we change the measure in Eq. $(i)$, substitute Wasserstein gradient flow in Eq. $(ii)$, apply integral by part in Eq. $(iii)$, leverage the chain rule in Eq. $(iv)$, and finally exploit the $G$-equivariance of the velocity field.
Next, we apply change of the measure reversely and apply integral by part again, then we have:
\begin{align*}
\frac{\partial}{\partial t} \int \xi(z) d \nu_t(z) &= -\int \nabla \xi(z)^\top \nabla \frac{\delta H}{\delta \mu}[\mu_t](z) d\nu_t(z) \\
&= \int \xi(z) \nabla \cdot \left(\nu_t \nabla \frac{\delta H}{\delta \mu}[\mu_t](z)\right) d\nu_t(z),
\end{align*}
which shows $\nu_t$ is following the exactly the same continuity equation with the same velocity field.
Now we conclude the proof by noticing the boundary condition $\nu_0 = g_{\#}\mu_0 = \mu_0$ since $\mu_0$ is $G$-invariant.
\end{proof}

%% file: tex/xx_prf_coord_descent.tex
\subsection{Proofs for Decoupled Dynamics}
\label{sec:prf_coord_descent}

\begin{theorem}
\label{thm:coord_descent_HP}
Consider $\mu_t$ follows Wasserstein gradient flow defined in Eqn. \ref{eqn:wass_pde} with Assumptions \ref{ass:app:init_mu}, \ref{ass:app:degree}, and \ref{ass:app:fo_var}, then each Hermite potential is optimized coordinate-wisely as:
\begin{align}
\label{eqn:coord_descent_HP}
\partial_t \rho_{r_i}[\mu_t] = - C_i(t) \partial_{\rho_{r_i}} L(\rho_t),
\end{align}
where $C_i(t) > 0$ is a time-dependent function. 
\end{theorem}
\begin{proof}
Due to Assumption \ref{ass:app:init_mu}, $\mu_0$ is Gaussian, thus, $\mu(A) = \mu(-A)$ for every measurable $A \subset M$.
By Lemma \ref{lem:inv_flow}, $\mu_t$ remains symmetric for $t \ge 0$.
Next we examine the dynamics of $H$ under the coordinates of HPs.
By Lemma \ref{lem:HP_flow}, we have the gradient flow of each HP:
\begin{align*}
\partial_t \rho_{r_i}[\mu_t] = -\sum_{j=1}^m \partial_{\rho_{r_j}} L[\mu_t] \left(\int \nabla r_i(z)^\top \nabla r_j(z) d\mu_t(z)\right).
\end{align*}
Let $G_{ij}(t) = \int \nabla r_i(z)^\top \nabla r_j(z) d\mu_t(z)$, then we analyze how $G_{ij}(t)$ evolves.
At the initialization, $\mu_0 = \gauss(0, I)$, this implies $G_{ij}(0) = 0$ for every $i \ne j$. This can be seen by first expanding $\nabla r_i(z)^\top \nabla r_j(z) = \sum_{k=1}^{m} h_{\alpha_i - e_k}(z) h_{\alpha_j - e_k}(z)$ as a series of Hermite polynomials with different indices, and use the fact that distinct Hermite polynomials are orthogonal under the Gaussian measure.
Recalling Lemma \ref{lem:G_flow}, we know that $\partial_t G_{ij}(t) = 0$, thus $G_{ij}(t) = 0$ for all $t \ge 0$.
While for the diagonal term: $G_{ii}(t) = \sum_{k=1}^{m} \int h_{\alpha_i - e_k}(z)^2 d\mu_t(z) \ge 0 := C_{i}(t)$.
Combining diagonal and crossing terms, we have $\partial_t \rho_{r_i}[\mu_t] = -\sum_{j=1}^m G_{ij} \partial_{\rho_{r_j}} L(\rho_t) = - G_{ii} \partial_{\rho_{r_i}} L(\rho_t)$ as desired.
\end{proof}

\begin{corollary}[Restatement of Theorem \ref{thm:coord_descent_MP}]
Under Assumptions \ref{ass:app:init_mu}, \ref{ass:app:degree}, and \ref{ass:app:fo_var}, each monomial potential is optimized coordinate-wisely as in Eq. \ref{eqn:coord_descent_HP}.
\end{corollary}
\begin{proof}
This is a straightforward result by Lemma \ref{MP_is_HP} and Theorem \ref{thm:coord_descent_HP}.
\end{proof}

\begin{lemma}
\label{lem:HP_flow}
Consider $\mu_t$ under the Wasserstein gradient flow defined in Eq. \ref{eqn:wass_pde}, the dynamics of HPs is given by the following trajectory:
\begin{align*}
\partial_t \rho_{r_i}[\mu_t] = -\sum_{j=1}^m \partial_{\rho_{r_j}} L(\rho_t) \left(\int \nabla r_i(z)^\top \nabla r_j(z) d\mu_t(z)\right).
\end{align*}
\end{lemma}
\begin{proof}
By expanding and simplifying the time derivatives of the HP potentials:
\begin{align*}
\partial_t \rho_{r_i}[\mu_t] &= \partial_t \E_{z \sim \mu_t}[r_i(z)] = \partial_t \int r_i(z) d\mu_t(z) \\
&\overset{(i)}{=} \int r_i(z) \nabla \cdot \left(\mu_t \nabla \frac{\delta H}{\delta \mu}[\mu_t](z) \right) dz \\
&\overset{(ii)}{=} -\int \nabla r_i(z)^\top \nabla \frac{\delta H}{\delta \mu}[\mu_t](z) d\mu_t(z) \\
&\overset{(iii)}{=} -\int \nabla r_i(z)^\top \left( \sum_{j=1}^{m} \partial_{\rho_{r_j}} L(\rho_t) \nabla r_j(z) \right) d\mu_t(z) \\
&\overset{(iv)}{=} -\sum_{j=1}^{m} \left( \int \nabla r_i(z)^\top \nabla r_j(z) d\mu_t(z) \right) \partial_{\rho_{r_j}} L(\rho_t),
\end{align*}
where we substitute Wasserstein gradient flow in Eq. $(i)$, Eq. $(ii)$ is due to integral by part, we apply Lemma \ref{lem:velocity} to obtain Eq. $(iii)$, and finally we re-organize terms to have Eq. $(iv)$.
\end{proof}

\begin{lemma}
\label{lem:G_flow}
Define $G_{ij}(t) = \int \nabla r_i(z)^\top \nabla r_j(z) d\mu_t(z)$, under Assumptions \ref{ass:app:init_mu}, \ref{ass:app:degree} and \ref{ass:app:fo_var}, $\partial_t G_{ij}(t) = 0$ for every $i \ne j$.
\end{lemma}
\begin{proof}
By taking the time derivative:
\begin{align*}
\partial_t G_{ij}(t) &= \partial_t \int \nabla r_i(z)^\top \nabla r_j(z) d\mu_t(z) \\
&\overset{(i)}{=} \int \left(\nabla r_i(z)^\top \nabla r_j(z)\right) \nabla \cdot \left(\mu_t \nabla \frac{\delta H}{\delta \mu}[\mu_t](z)\right) dz \\
&\overset{(ii)}{=} -\int \nabla \left(\nabla r_i(z)^\top \nabla r_j(z)\right)^\top \nabla \frac{\delta H}{\delta \mu}[\mu_t](z) d\mu_t(z) \\
&\overset{(iii)}{=} -\int \nabla \left(\nabla r_i(z)^\top \nabla r_j(z)\right)^\top \left( \sum_{k=1}^{m} \partial_{\rho_{r_k}} L(\rho_t) \nabla r_k(z) \right) d\mu_t(z) \\
&= -\sum_{k=1}^{m} \partial_{\rho_{r_k}} L(\rho_t) \int \nabla \left(\nabla r_i(z)^\top \nabla r_j(z)\right)^\top \nabla r_k(z) d\mu_t(z),
\end{align*}
where we plug the Wasserstein gradient flow into Eq. $(i)$, apply integral by part in Eq. $(ii)$, and recall Lemma \ref{lem:velocity} in Eq. $(iii)$.
Next we first that $r_i = h_{\alpha_i}$ is a Hermite polynomial with multi-index $\alpha_i$ and then use the fact of Hermite polynomials that $\nabla h_{\alpha}(z) = \sum_{l=1}^{d} \alpha_l h_{\alpha - e_l} e_l$.
First of all, consider the term $\nabla r_i(z)^\top \nabla r_j(z)$:
\begin{align*}
\nabla r_i(z)^\top \nabla r_j(z) = \sum_{l=1}^{d} \alpha_{i, l} \alpha_{j, l}  h_{\alpha_i - e_l}(z) h_{\alpha_j - e_l}(z),
\end{align*}
and its gradient according to product rule is:
\begin{align*}
\nabla \left(\nabla r_i(z)^\top \nabla r_j(z) \right) = \sum_{l=1}^{d} \alpha_{i, l} \alpha_{j, l}  \nabla h_{\alpha_i - e_l}(z) h_{\alpha_j - e_l}(z) + \sum_{l=1}^{d} \alpha_{i, l}(z) \alpha_{j, l} h_{\alpha_i - e_l}(z) \nabla h_{\alpha_j - e_l}(z).
\end{align*}
Moving forward, we consider the combined term: $\nabla \left(\nabla r_i(z)^\top \nabla r_j(z)\right)^\top \nabla r_k(z)$:
\begin{align*}
\nabla \left(\nabla r_i(z)^\top \nabla r_j(z)\right)^\top \nabla r_k(z) &= \sum_{p=1}^{d}\sum_{l=1}^{d} \alpha_{i, l} \alpha_{j, l} (\alpha_i - e_l)_p h_{\alpha_i - e_l - e_p} h_{\alpha_j - e_l} h_{\alpha_k - e_p}(z) \\
& \quad + \sum_{p=1}^{d}\sum_{l=1}^{d} \alpha_{i, l} \alpha_{j, l} (\alpha_j - e_l)_p h_{\alpha_i - e_l} h_{\alpha_j - e_l - e_p} h_{\alpha_k - e_p}(z)
\end{align*}
Under Assumption \ref{ass:app:fo_var}, we have that $\mu_t$ is reflection symmetric for all $t \ge 0$, then we use Lemma \ref{lem:triple_product} to show that: $\int h_{\alpha_i - e_l - e_p} h_{\alpha_j - e_l} h_{\alpha_k - e_p}(z) d\mu_t(z) = 0$ for every $k \in [m]$ and $l, p \in [d]$.
As a consequence, we affirm $\partial_t G_{i, j}(t) = 0$.
\end{proof}

\begin{lemma}
\label{lem:triple_product}
Suppose $\alpha, \beta, \gamma \in \integer^d$ and for any canonical basis $e_1, e_2 \in \{0, 1\}^d$, under Assumption \ref{ass:app:degree}, we have $\int h_{\alpha - e_1 - e_2}(z) h_{\beta - e_1}(z) h_{\gamma - e_2}(z) d\mu(z) = 0$ for any reflection symmetric $\mu$.
\end{lemma}
\begin{proof}
First of all, we note that if a function $f$ is odd, $f$ has zero integral against symmetric measures (i.e., $\mu(-A) = \mu(A)$, for every measurable $A \subset M$) since:
\begin{align*}
\int -f(z) d\mu(z) = \int f(-z) d\mu(z)  = \int f(z) d\mu(-z) = \int f(z) d\mu(z).
\end{align*}
By the property of Hermite polynomials:
\begin{align*}
h_{\alpha - e_1 - e_2}(-z) h_{\beta - e_1}(-z) h_{\gamma - e_2}(-z) = (-1)^{\lVert \alpha\rVert_1 + \lVert\beta\rVert_1 + \lVert\gamma\rVert_1 - 4} h_{\alpha - e_1 - e_2}(z) h_{\beta - e_1}(z) h_{\gamma - e_2}(z).
\end{align*}
Due to Assumption \ref{ass:app:degree}, we notice that $\lVert \alpha\rVert_1 + \lVert\beta\rVert_1 + \lVert\gamma\rVert_1 - 4$ is odd, which concludes the proof by verifying the product of the triplet is an odd function.
\end{proof}

%% file: tex/xx_prf_dim_reduction.tex
\subsection{Proofs for Effective Degree of Freedom Reduction}

\begin{theorem}
[Restatement of Theorem \ref{thm:low_ent}]
For every $\varrho \in \real^m$, consider a measure $\mu$ that realizes $\rho_{r_i}[\mu] = \varrho_i$ for every $i \in [m]$ with minimal differential entropy $E[\mu] := \mean_{\mu}[-\log d\mu]$. Then all such measures form a Riemannian submanifold of $P(M)$, whose dimension is at most $m$.
\end{theorem}
\begin{proof}
We denote $u(z) = \frac{d \mu(z)}{d z}$ as the density function of $\mu$.
Given $\varrho \in \real^m$, we consider the following constrained functional optimization problem:
\begin{align*}
\min_{u} \int -\log u(z) u(z) dz,
\quad \text{s.t.} \quad \int r_i(z) u(z) dz = \varrho_i, \forall i \in [m],
\quad \int u(z) dz = 1.
\end{align*}
Using the Lagrangian multiplier method, we define:
\begin{align*}
\mathscr{L}[\mu] = \int -\log u(z) u(z) dz + \sum_{i=1}^{m} \lambda_i \left( \int r_i(z) u(z) dz - \varrho_i \right) + \eta \left( \int u(z) dz - 1 \right),
\end{align*}
where $\lambda \in \real^{m}$ and $\eta \in \real$ is Lagrangian multipliers.
By KKT condition , we set first variation to zero:
\begin{align*}
\frac{\delta \mathscr{L}}{\delta u}(z) = - (1 + \log u(z)) + \sum_{i=1}^{m} \lambda_i r_i(z)+ \eta = 0,
\end{align*}
which gives $\log u(z) = \sum_{i=1}^{m} \lambda_i r_i(z) + \eta - 1$ and henceforth the density function takes the form as below:
\begin{align*}
u^*(z) = \exp\left(\sum_{i=1}^{m} \lambda_i r_i(z) + \eta - 1\right).
\end{align*}
To satisfy primal feasibility $\int u(z) dz = 1$, we let $\eta = -\log(Z) + 1$ where $Z = \int u^*(z) dz$, and the density function becomes:
\begin{align*}
u^*(z) = \frac{1}{Z}\exp\left(\sum_{i=1}^{m} \lambda_i r_i(z)\right),
\end{align*}
where $\lambda$ is chosen to satisfy $\int r_i(z) u(z) dz = \varrho_i$ for every $i \in [m]$.

By what is shown above, we note that for every $\varrho \in \real^m$, the corresponding minimizer $u^*$ belongs to an exponential family parameterized by $m$ parameters $\lambda \in \real^m$.
Therefore, all $\mu^*$ form a Riemannian manifold of dimension at most $m$ \citep{amari2000methods}.
\end{proof}

\begin{remark}
\label{rmk:kl_sol}
Consider measures $\mu$ realizing MP constraints $\rho_{r_i}[\mu] = \varrho_i$ for every $i \in [m]$ and some $\varrho_i \in \real^m$, while minimizing Kullback–Leibler divergence $\mathcal{D}_{KL}(\mu \Vert \nu)$ or $\alpha$-divergence $\mathcal{D}_{\alpha}(\mu \Vert \nu)$ with some target measure $\nu$
All such measures also form a Riemannian submanifold of dimension at most $m$.
This is because the minimizers of KL-divergence with MP constraints take the form of $u^*(z) \propto \exp(\sum_{i=1}^{m} \lambda_i r_i(z)) v(z)$ and those minimizers for $\alpha$-divergence can be written as $u^*(z) \propto [1 - (1-\alpha)\sum_{i=1}^{m} \lambda_i r_i(z)]_{+}^{1/(1-\alpha)} v(z)$ for $0 < \alpha < 1$ or $\alpha > 1$, where $[\cdot]_{+}$ truncates at zero, $u^*$ and $v$ denotes the density function of the optimal measure $\mu^*$ and target measure $\nu$.
\end{remark}

\paragraph{Notation Shorthand.}
Starting from now, we need the following definitions to simplify notations.
Let $\Delta(t) = \nabla L(\rho_t) \in \real^m$ be the gradient of $L$ in terms of $\rho$, i.e., $\Delta(t)_i = \partial_i L(\rho_t)$ and $A(t) = \nabla^2 L(\rho_t) \in \real^{m \times m}$ be the Hessian of $L$, i.e., $A(t)_{ij} = \frac{\partial L}{\partial \rho_i \partial \rho_j}$, $K(t) \in \real^{m \times m}$ be a kernel matrix defined element-wisely as: $K(t)_{ij} = \int r_i(z) r_j(z) d\mu_t(z)$.
According to Lemma \ref{lem:eigen_in_R}, we can represent an eigenvector associated with a non-zero eigenvalue as $v_t = \sum_{j = 1}^{m} q(t)_j r_j$ for a time-varying vector $q(t) \in \real^m$ and denote $\dot{q}(t) \in \real^m$ as its time derivative.

\begin{theorem}[Restatement of Theorem \ref{thm:crit_times}]
Consider loss functional in Eq. \ref{eqn:H_form}, then all eigenfunctions of second variation $\mathbb{L}_t$ lie in a subspace spanned by the monomial set $\Set{R}$, i.e, $v_i \subset \Span(\Set{R})$.
Moreover, under Assumptions \ref{ass:app:init_mu}, \ref{ass:app:degree}, \ref{ass:app:fo_var}, and $[\nabla^3 L(\rho_t)]_i \nabla_i L \succeq 0$ for every $i \in [m]$, then $\mathbb{L}_t$ will have non-increasing $\frac{\partial}{\partial t} \lambda(t) \le 0$.
There will be finitely many $0 < t_1 < t_2 < \cdots < t_m$ where an eigenvalue of $\mathbb{L}_t$ crosses zero.
\end{theorem}
\begin{proof}
Lemma \ref{lem:eigen_in_R} directly shows our first claim: $v \in \Set{R}$ if its corresponding eigenvalue is non-zero.
To show the second part, we use Lemma \ref{lem:eigen_pde} to describe the trajectory of $\lambda(t)$ under Wasserstein gradient flow as:
\begin{align*}
\frac{\partial}{\partial t} \lambda(t)
&=\left\langle\frac{\partial}{\partial t} 
 \mathbb{L}_t[v_t], v_t\right\rangle_{L(\mu_t)} - \lambda(t) \left\langle\frac{\partial}{\partial t} 
 v_t, v_t\right\rangle_{L(\mu_t)} \\
&\overset{(*)}{=} [K(t)q(t)]^\top \left[\frac{\partial}{\partial t} \nabla^2 L\right] [K(t)q(t)] - 2\lambda(t) \left\langle \frac{\partial}{\partial t} v_t(z), v_t(z) \right\rangle_{L(\mu_t)},
\end{align*}
where Eq. $(*)$ is due to Lemma \ref{lem:dyna_sec_var}.
By Lemma \ref{lem:zero_prod_dot_v_v}, we cancels all the terms $\left\langle \frac{\partial}{\partial t} v_t(z), v_t(z) \right\rangle_{L(\mu_t)} = 0$ under Assumptions \ref{ass:app:init_mu}, \ref{ass:app:degree}, and \ref{ass:app:fo_var}.
This gives $\frac{\partial}{\partial t} \lambda(t)
= [K(t)q(t)]^\top \left[\frac{\partial}{\partial t} \nabla^2 L\right] [K(t)q(t)]$.
We notice that $\frac{\partial}{\partial t} \nabla^2 L = \sum_{k=1}^{m} [\nabla^3 L(\rho_t)]_k \partial_t \rho(t) = -\sum_{k=1}^{m} [\nabla^3 L(\rho_t)]_k C_k(t) \nabla_k L$ by Theorem \ref{thm:coord_descent_HP}.
With assumption that $[\nabla^3 L(\rho_t)]_i \nabla_i L \succeq 0$ for every $i \in [m]$, $\frac{\partial}{\partial t} \nabla^2 L \preceq 0$, we know that $\frac{\partial}{\partial t} \nabla^2 L \preceq 0$, which implies $\frac{\partial}{\partial t} \lambda(t) \le 0$.
Finally, we point out that $\mathbb{L}_t$ is a rank-$m$ operator as its non-degenerate eigenspaces are all spanned by $\Set{R}$.
Since eigenvalues are non-increasing, once it crosses zero, it stays negative.
Therefore, there are at most $m$ times when an eigenvalue $\lambda(t)$ crosses zero.
\end{proof}

\begin{lemma}
\label{lem:eigen_in_R}
Consider loss functional in Eq. \ref{eqn:H_form} with the monomial set $\Set{R}$, for every eigenfunction $v_t \in T_{\mu_t}P(M)$ of $\mathbb{L}_t$, if $v_t \notin \ker(\mathbb{L}_t)$, then $v_t \in \Span(\Set{R})$
\end{lemma}
\begin{proof}
By Lemma \ref{lem:second_var} and the definition of eigenfunction, we have:
\begin{align*}
\mathbb{L}_t[v_t] = \sum_{i,j=1}^m \partial^2_{ij} L(\rho_t) \left( \int r_i(z) v_t(z) d\mu_t(z) \right) r_j = \lambda(t) v_t,
\end{align*}
where $\lambda(t)$ is the associated eigenvalue.
When $\lambda(t) \ne 0$, $v_t = \lambda(t)^{-1} \sum_{i,j=1}^m \partial^2_{ij} L(\rho_t) \left( \int r_i v_t d\mu_t \right) r_j$ becomes a linear combination of $r_j$'s.
This suffices to conclude the proof.
\end{proof}

\begin{lemma}[Characterization of Eigenvalue Dynamics] \label{lem:eigen_pde}
Given second variation $\mathbb{L}_t$, consider one of its eigenvalue $\lambda(t)$ and its eigenfunction $v_t$, we have
\begin{align*}
\frac{\partial}{\partial t} \lambda(t)
=\left\langle\left[\frac{\partial}{\partial t} \mathbb{L}_t\right][v_t],  v_t \right\rangle_{L(\mu_{t})} = \left\langle\frac{\partial}{\partial t} \mathbb{L}_t \left[v_t\right], v_t \right\rangle_{L(\mu_{t})} - \left\langle \frac{\partial}{\partial t} v_t(z), v_t(z) \right\rangle_{L(\mu_t)}.
\end{align*}
\end{lemma}
\begin{proof}
Differentiate $\mathbb{L}_t[v_t]=\lambda(t) v_t$ through time:
\begin{align*}
\frac{\partial}{\partial t} \mathbb{L}_t[v_t] + \mathbb{L}_t \left[\frac{\partial}{\partial t} v_t\right] = \frac{\partial}{\partial t} \lambda(t) v_t + \lambda(t) \frac{\partial}{\partial t} v_t
\end{align*}
Next we take inner product of $v_t$ at both sides:
\begin{align*}
0 &= \int \left[\frac{\partial}{\partial t} \mathbb{L}_t[v_t]\right](z) v_t(z) d\mu_t(z) - \int \frac{\partial}{\partial t} \lambda(t) v_t(z)^2 d\mu_t(z) - \int \lambda(t) \frac{\partial}{\partial t} v_t(z) v_t(z) d\mu_t(z) \\
&= \int \left[\frac{\partial}{\partial t} \mathbb{L}_t\right][v_t](z) v_t(z) d\mu_t(z) + \int 
\mathbb{L}_t\left[\frac{\partial}{\partial t} v_t\right](z) v_t(z) d\mu_t(z) - \int \frac{\partial}{\partial t} \lambda(t) v_t(z)^2 d\mu_t(z) \\ &\quad - \int \lambda(t) \frac{\partial}{\partial t} v_t(z) v_t(z) d\mu_t(z) \\
&\overset{(*)}{=} \int \left[\frac{\partial}{\partial t} \mathbb{L}_t\right][v_t](z) v_t(z) d\mu_t(z) + \int \mathbb{L}_t\left[v_t\right](z) \frac{\partial}{\partial t} v_t(z) d\mu_t(z) - \int \frac{\partial}{\partial t} \lambda(t) v_t(z)^2 d\mu_t(z) \\&\quad - \int \lambda(t) \frac{\partial}{\partial t} v_t(z) v_t(z) d\mu_t(z) \\
&\overset{(**)}{=} \int \left[\frac{\partial}{\partial t} \mathbb{L}_t\right][v_t](z) v_t(z) d\mu_t(z) + \int  \lambda(t) v_t(z) \frac{\partial}{\partial t} v_t(z) d\mu_t(z) - \int \frac{\partial}{\partial t} \lambda(t) v_t(z)^2 d\mu_t(z) \\&\quad - \int \lambda(t) \frac{\partial}{\partial t} v_t(z) v_t(z) d\mu_t(z),
\end{align*}
where in RHS of Eq. $(*)$, we use the self-adjointness of $\mathbb{L}_t$, and in RHS of Eq. $(**)$, we notice that $v_t$ is the eigenfunction of $\mathbb{L}_t$.
Next we notice that the second and fourth terms in Eq. $(**)$ cancel with each other and use the fact that $v_t$ are normalized against $\mu_t$: $\int v_t(z)^2 d\mu_t = 1$, then henceforth, we can simplify the above equation to:
\begin{align*}
\int \left[\frac{\partial}{\partial t} \mathbb{L}_t \right][v_t](z) v_t(z) d\mu_t(z) = \left[\frac{\partial}{\partial t} \lambda(t) \right] \int  v_t(z)^2 d\mu_t(z) = \frac{\partial}{\partial t} \lambda(t),
\end{align*}
as desired.
\end{proof}

\begin{lemma}\label{lem:gague}
Given second variation $\mathbb{L}_t$ of $H[\mu_t]$, for every eigenvalue $\lambda(t)$ and eigenfunction $v_t$ of $\mathbb{L}_t$, we have the following equalities:
\begin{align*}
2\int \frac{\partial}{\partial t} v_t(z) v_t(z) d\mu_t(z) = -\int v_t(z)^2 d (\partial_t \mu_t(z))
\end{align*}
\end{lemma}
\begin{proof}
We use the fact that $\int v_t(z) v_t(z) d\mu_t = 1$ for every $t \ge 0$.
Therefore, its time derivatives constantly equals zero:
\begin{align*}
0 = \frac{\partial}{\partial t}\int v_t(z) v_t(z) d\mu_t(z) =
\int 2 v_t(z) \frac{\partial}{\partial t} v_t(z) d\mu_t(z) + \int v_t(z)^2 d(\partial_t \mu_t(z)) dz.
\end{align*}
In particular, $\int v_t(z) \frac{\partial}{\partial t} v_t(z) d\mu_t(z) = 0$ if $\mu_t$ is time invariant.
\end{proof}

\begin{lemma}
\label{lem:inner_prod_Gq}
The following facts hold:
\begin{align*}
[K(t)q(t)]_i = \int r_i(z) v_t(z) d\mu_t(z), \quad 
[G\dot{q}(t)]_i = \int r_i(z) \frac{\partial}{\partial t} v_t(z) d\mu_t(z)
\end{align*}
\end{lemma}
\begin{proof}
By simple algebra:
\begin{align*}
\int r_i(z) v_t(z) d\mu_t(z) &= \int r_i(z) \left(\sum_{j=1} q(t)_j r_j(z)\right) d\mu_t(z) \\
&= \sum_{j=1} \left(\int r_i(z) r_j(z) d\mu_t(z)\right) q(t)_j = [K(t)q(t)]_i.
\end{align*}
And similarly, we have $\int r_i(z) \frac{\partial}{\partial t} v_t(z) d\mu_t(z) = [G\dot{q}]_i$ by exchanging summation and partial differential operators.
\end{proof}

\begin{lemma}
\label{lem:eigen_AGq}
It holds that $A(t) K(t) q(t) = \lambda(t) q(t)$ for $\lambda(t) \ne 0$.
\end{lemma}
\begin{proof}
We leverage the operator form in Lemma \ref{lem:second_var}:
\begin{align*}
\mathbb{L}_t[v_t] = \sum_{i,j=1}^m \partial^2_{ij} L(\rho_t) \left( \int r_i(z) v_t(z) d\mu_t(z) \right) r_j = \lambda(t) v_t.
\end{align*}
By Lemma \ref{lem:inner_prod_Gq}, we have $\sum_{i,j=1}^m A(t)_{ij} [K(t)q(t)]_i r_j = \lambda(t) v_t$, which introduces a linear representation of $v_t$ under $r_j$'s.
Due to the linear independence of monic monomials, the linear representation of $v_t$ under $\Set{R}$ is unique (by Lemma \ref{lem:eigen_in_R}).
Hence, $q(t)_j = \lambda(t)^{-1} \sum_{i=1}^{m} A(t)_{ij} [K(t)q(t)]_i$ for every $j \in [m]$.
This is $\lambda(t) q(t) = A(t) K(t) q(t)$ in matrix form.
\end{proof}

\begin{lemma}
\label{lem:dyna_sec_var}
The following equality holds:
\begin{align*}
\left\langle\frac{\partial}{\partial t} 
 \mathbb{L}_t[v_t], v_t\right\rangle_{L(\mu_t)} = [K(t)q(t)]^\top \left[\frac{\partial}{\partial t} \nabla^2 L\right] [K(t)q(t)] -\lambda(t) \left\langle \frac{\partial}{\partial t} v_t(z), v_t(z) \right\rangle_{L(\mu_t)}.
\end{align*}
\end{lemma}
\begin{proof}
By Lemma \ref{lem:second_var}, for any $g \in T_{\mu_t}P(M)$
\begin{align*}
\frac{\partial}{\partial t} \mathbb{L}_t[g] &= \frac{\partial}{\partial t} \sum_{i,j=1}^m \partial^2_{ij} L(\rho_t) \left( \int r_i(z) g(z) d\mu_t(z) \right) r_j \\
&= \underbrace{\sum_{i,j=1}^m \left[\frac{\partial}{\partial t} \nabla^2 L(\rho_t)\right]_{ij} \left( \int r_i(z) g(z) d\mu_t(z) \right) r_j(z)}_{\mathbb{O}_1[g]} \\
&\quad + \underbrace{\sum_{i,j=1}^m \partial^2_{ij} L(\rho_t) \frac{\partial}{\partial t}\left[\left( \int r_j(z) g(z) d\mu_t(z) \right) r_i(z) \right]}_{\mathbb{O}_2[g]}
\end{align*}
We analyze term $\mathbb{O}_2$ first by:
\begin{align*}
\frac{\partial}{\partial t}\left[\left( \int r_j(z) g(z) d\mu_t(z) \right) r_i \right] &= \frac{\partial}{\partial t}\left[ \int r_j(z) g(z) d\mu_t(z) \right]. r_i(z)
\end{align*}
Next, we expand the time derivatives of the integral:
\begin{align*}
&\frac{\partial}{\partial t}\left[ \int r_i(z) g(z) d\mu_t(z) \right] \overset{(*)}{=} \int r_i(z) \frac{\partial}{\partial t} g(z) d\mu_t(z) + \int r_i(z) g(z) \nabla \cdot \left(\mu_t(z) \nabla\frac{\delta H}{\delta \mu}(z)\right) dz \\
&\overset{(**)}{=} \int r_i(z) \frac{\partial}{\partial t} g(z) d\mu_t(z) - \int \nabla\left[r_i(z) g(z)\right]^\top \nabla\frac{\delta H}{\delta \mu}[\mu_t](z) d\mu_t(z) \\
&\overset{(***)}{=} \int r_i(z) \frac{\partial}{\partial t} g(z) d\mu_t(z) - \int g(z) \nabla r_i(z)^\top \nabla\frac{\delta H}{\delta \mu}[\mu_t](z) d\mu_t(z) - \int r_i(z) \nabla g(z)^\top \nabla\frac{\delta H}{\delta \mu}[\mu_t](z) d\mu_t(z),
\end{align*}
where we plug the PDE of Wasserstein gradient flow into Eq. $(*)$ and utilize the integral by part in Eqs. $(**)$ and $(***)$.
Henceforth, we can expand $\mathbb{O}_2[g]$ with $g = v_t$ and take inner product with $v_t$:
\begin{align*}
\int \mathbb{O}_2[v_t](z) v_t(z) d\mu_t(z) &= \sum_{i,j=1}^m \partial^2_{ij} L(\rho_t) \left( \int r_i(z) v_t(z) d\mu_t(z) \right) \frac{\partial}{\partial t}\left[\left( \int r_i(z) v_t(z) d\mu_t(z) \right) r_j(z) \right] \\
&= \Omega - \Xi - \Upsilon,
\end{align*}
where we define:
\begin{align*}
\Omega &= \sum_{i,j=1}^m \partial^2_{ij} L(\rho_t) \left( \int r_i(z) v_t(z) d\mu_t(z) \right) \left[ \int r_j(z) \frac{\partial}{\partial t} v_t(z) d\mu_t(z) \right], \\
\Xi &= \sum_{i,j=1}^m \partial^2_{ij} L(\rho_t) \left( \int r_i(z) v_t(z) d\mu_t(z) \right) \left[ \int v_t(z) \nabla r_j(z)^\top \nabla\frac{\delta H}{\delta \mu}[\mu_t](z) d\mu_t(z) \right], \\
\Upsilon &= \sum_{i,j=1}^m \partial^2_{ij} L(\rho_t) \left( \int r_i(z) v_t(z) d\mu_t(z) \right) \left[ \int r_j(z) \nabla v_t(z)^\top \nabla\frac{\delta H}{\delta \mu}[\mu_t](z) d\mu_t(z) \right].
\end{align*}
By Lemma \ref{lem:Omg_Xi_Ups}, we find that $\Omega - \Xi - \Upsilon = -\lambda(t) \left\langle \frac{\partial}{\partial t} v_t(z), v_t(z) \right\rangle_{L(\mu_t)}$.

Now we examine $\mathbb{O}_1$. By substituting $\mathbb{O}_1[g]$ with $g = v_t$ and taking inner product with $v_t$, we have:
\begin{align*}
\int \mathbb{O}_1[v_t](z) v_t(z) d\mu_t(z) = \sum_{i,j=1}^m \left[\frac{\partial}{\partial t} \nabla^2 L(\rho_t)\right]_{ij} \left( \int r_i(z) g(z) d\mu_t(z) \right) \left( \int r_j(z) g(z) d\mu_t(z) \right)
\end{align*}
By Lemma \ref{lem:inner_prod_Gq}, we write above equation as $[K(t)q(t)]^\top \frac{\partial}{\partial t} \nabla^2 L(\rho_t) [K(t)q(t)]$.
Combining $\langle\mathbb{O}_1[v_t], v_t\rangle_{L(\mu_t)}$ and $\langle\mathbb{O}_2[v_t], v_t\rangle_{L(\mu_t)}$, we obtain the desired result.
\end{proof}

\begin{lemma}
\label{lem:Omg_Xi_Ups}
It holds that $\Xi = \Upsilon$ and $\Omega - \Xi - \Upsilon = -\lambda(t) \left\langle \frac{\partial}{\partial t} v_t(z), v_t(z) \right\rangle_{L(\mu_t)}$.
\end{lemma}
\begin{proof}
First of all, we examine the term $\Omega$:
\begin{align*}
\Omega &= \sum_{i,j=1}^m \partial^2_{ij} L(\rho_t) \left( \int r_i(z) v_t(z) d\mu_t(z) \right) \left[ \int r_i(z) \frac{\partial}{\partial t} v_t(z) d\mu_t(z) \right] \\
&\overset{(*)}{=} \sum_{i,j=1}^{m} A(t)_{ij} [K(t) q(t)]_i [K(t) \dot{q}(t)]_j = \dot{q}(t)^\top K(t)^\top A(t) K(t) q(t) \overset{(**)}{=} \lambda(t) \dot{q}(t)^\top K(t) q(t),
\end{align*}
where Eq. $(*)$ is due to Lemma \ref{lem:inner_prod_Gq} and Eq. $(**)$ is because of Lemma \ref{lem:eigen_AGq}.
Further on, we can establish the following relation to Lemma \ref{lem:gague}:
\begin{align*}
(\star) \quad  &\int \frac{\partial}{\partial t} v_t(z) v_t(z) d\mu_t(z) = \int \left(\sum_{i=1} \frac{\partial}{\partial t} q(t) r_i(z) \right) \left(\sum_{j=1} q(t) r_j(z) \right) d\mu_t(z) \\
&= \sum_{i,j=1} \dot{q}(t)_i  q(t)_j \left(\int r_i(z) r_j(z) d\mu_t(z)\right) = \dot{q}(t)^\top K(t) q(t) = \frac{1}{\lambda(t)}\Omega.
\end{align*}
Now we examine $\Xi$ and $\Upsilon$.
Using Lemma \ref{lem:inner_prod_Gq} and \ref{lem:velocity}, we can expand $\Xi$ as:
\begin{align*}
\Xi &= \sum_{i,j=1}^m \partial^2_{ij} L(\rho_t) \left( \int r_i(z) v_t(z) d\mu_t(z) \right) \left[ \int v_t(z) \nabla r_j(z)^\top \nabla\frac{\delta H}{\delta \mu}[\mu_t](z) d\mu_t(z) \right] \\
&= \sum_{i,j=1}^{m} A(t)_{ij} [K(t) q(t)]_i \left[ \int \left(\sum_{k=1}^m q(t)_k r_k(z)\right) \nabla r_j(z)^\top \left(\sum_{l=1}^m \Delta(t)_l \nabla r_l(z)\right) d\mu_t(z) \right] \\
&= \sum_{i,j=1}^{m} A(t)_{ij} [K(t) q(t)]_i (q(t)^\top B_j \Delta),
\end{align*}
where $B_{j,k,l} = \int r_k(z) \nabla r_j(z)^\top \nabla r_l(z) d\mu_t(z)$.
Rewriting to the matrix form and by Lemma \ref{lem:eigen_AGq},
\begin{align*}
\Xi &= [\cdots q(t)^\top B_j \Delta \cdots]^\top A(t) K(t) q(t) \\
&= \lambda(t) [\cdots q(t)^\top B_j \Delta \cdots]^\top q(t) \\
&= \lambda(t) \sum_{i,j,k=1}^{m} q(t)_i q(t)_j \Delta(t)_k \int r_j(z) \nabla r_i(z)^\top \nabla r_k(z) d\mu_t(z).
\end{align*}
Similarly, applying Lemma \ref{lem:inner_prod_Gq} and \ref{lem:velocity} to the term $\Upsilon$:
\begin{align*}
\Upsilon &= \sum_{i,j=1}^m \partial^2_{ij} L(\rho_t) \left( \int r_i(z) v_t(z) d\mu_t(z) \right) \left[ \int v_t(z) \nabla r_j(z)^\top \nabla\frac{\delta H}{\delta \mu}[\mu_t](z) d\mu_t(z) \right] \\ &\quad = \sum_{i,j=1}^{m} A(t)_{ij} [K(t) q(t)]_i (q(t)^\top C_j \Delta),
\end{align*}
where $C_{j,k,l} = \int r_j(z) \nabla r_k(z)^\top \nabla r_l(z) d\mu_t(z)$.
Likewise, using Lemma \ref{lem:eigen_AGq}, we have:
\begin{align*}
\Upsilon &= \lambda(t) \sum_{i,j,k=1}^{m} q(t)_i q(t)_j \Delta(t)_k \int r_i(z) \nabla r_j(z)^\top \nabla r_k(z) d\mu_t(z).
\end{align*}
By far, we can observe that $\Xi = \Upsilon$ by symmetry.
Next, we establish the following relation to Lemma \ref{lem:gague}:
\begin{align*}
(\star\star) \quad & \int v_t(z)^2 d(\partial_t \mu_t(z)) dz = \int v_t(z)^2 \nabla \cdot \left(\mu_t(z) \nabla\frac{\delta H}{\delta \mu}(z)\right) dz \\ &\overset{(*)}{=} -\int \nabla[v_t(z)^2]^\top \nabla\frac{\delta H}{\delta \mu}(z) d\mu_t(z) = -2 \int v_t(z) \nabla v_t(z)^\top \nabla\frac{\delta H}{\delta \mu}(z) d\mu_t(z) \\
&\overset{(**)}{=} -2 \int \left(\sum_{i=1}^{m} q(t)_i r_i(z) \right) \left(\sum_{j=1}^{m} q(t)_j \nabla r_j(z) \right)^\top \left(\sum_{k=1}^{m} \Delta(t)_k \nabla r_k(z) \right) d\mu_t(z) \\
&= -2 \sum_{i,j,k=1}^{m}  q(t)_i  q(t)_j \Delta(t)_k \left(\int r_i(z)\nabla r_j(z)^\top  \nabla r_k(z)  d\mu_t(z) \right) = -\frac{1}{\lambda(t)} (\Xi + \Upsilon),
\end{align*}
where Eq. $(*)$ leverages integral by part and we rewrite $v_t$ using $r_i$'s by Lemma \ref{lem:eigen_in_R} in Eq. $(**)$. 
We complete the proof by combining Eq. $(\star)$ and $(\star\star)$ with Lemma \ref{lem:gague}:
\begin{align*}
\Omega - \Xi - \Upsilon &= \lambda(t) \int \frac{\partial}{\partial t} v_t(z) v_t(z) d\mu_t(z) + \lambda(t) \int v_t(z)^2 \partial_t \mu_t(z) dz \\
&= \lambda(t) \left[ \int \frac{\partial}{\partial t} v_t(z) v_t(z) d\mu_t(z) + \int v_t(z)^2 \partial_t \mu_t(z) dz \right] \\
&= -\lambda(t) \int \frac{\partial}{\partial t} v_t(z) v_t(z) d\mu_t(z),
\end{align*}
which completes the proof.
\end{proof}

\begin{lemma}
\label{lem:zero_prod_dot_v_v}
Under Assumptions \ref{ass:app:init_mu}, \ref{ass:app:degree}, and \ref{ass:app:fo_var}, it holds that $\left\langle \frac{\partial}{\partial t} v_t(z), v_t(z) \right\rangle_{L(\mu_t)} = 0$.
\end{lemma}
\begin{proof}
Instead of directly examine $\left\langle \frac{\partial}{\partial t} v_t(z), v_t(z) \right\rangle_{L(\mu_t)}$, we leverage Lemma \ref{lem:gague} to inspect the following equation:
\begin{align*}
\left\langle \frac{\partial}{\partial t} v_t(z), v_t(z) \right\rangle_{L(\mu_t)} = \sum_{i,j,k=1}^{m}  q(t)_i  q(t)_j \Delta(t)_k \left(\int r_i(z)\nabla r_j(z)^\top  \nabla r_k(z)  d\mu_t(z) \right).
\end{align*}
We narrow our analysis to the following part:
\begin{align*}
\int r_i(z) \nabla r_j(z)^\top  \nabla r_k(z)  d\mu_t(z) &\overset{(*)}{=} \int h_{\alpha_i}(z) \sum_{l=1}^{d} \alpha_{j, l} \alpha_{k, l} h_{\alpha_j - e_l}(z) h_{\alpha_k - e_l}(z) d\mu_t(z) \\
&= \sum_{l=1}^{d} \int \alpha_{j, l} \alpha_{k, l} h_{\alpha_i}(z)  h_{\alpha_j - e_l}(z) h_{\alpha_k - e_l}(z) d\mu_t(z),
\end{align*}
where in Eq. $(*)$, we use the analytic form of the derivative of Hermite polynomials.
Next, we do degree counting similar to Lemma \ref{lem:triple_product}.
Under Assumption \ref{ass:app:degree}, $\lVert  \alpha_i \rVert_1 + \lVert  \alpha_j \rVert_1 + \lVert  \alpha_k \rVert_1 - 2$ is odd, thus $h_{\alpha_i}(z)  h_{\alpha_j - e_l}(z) h_{\alpha_k - e_l}(z)$ is an odd function.
Under Assumptions \ref{ass:app:init_mu} and \ref{ass:app:fo_var}, we know that $\mu_t$ remains symmetric by Lemma \ref{lem:inv_flow}.
Taking the integral of an odd function against a symmetric measure yields zero, which completes the proof.
\end{proof}

%% file: tex/xx_prf_sample_complex.tex
\subsection{Formal Statement and Proofs for Sample Complexity of Invariant Learning}
\label{sec:prf_sample_complex}

In this section, we give detailed mathematical setup and statements for Theorem \ref{thm:sample_complex}. 
Our notation, problem formulation, and proof are directly based on the seminal works by \cite{tahmasebi2023exact}.
We apply standard PAC analysis in addition.
\paragraph{Notations.}
Suppose $M_d \subset \real^d$ is a smooth connected compact boundaryless data manifold of dimension $\dim M_d$.
Let $G$ denote a compact Lie group of dimension $\dim(G)$, acting smoothly on the manifold $M_d$.
We represent $M_d / G$ as the quotient space of $M_d$ by actions in $G$.
We denote $L^2(M_d)$ as a family of square-integrable functions on $M_d$ and $L_{inv}^2(M_d, G) \subseteq L^2(M_d)$ as the $G$-invariance functions in $L^2(M_d)$, i.e. for every $f \in L_{inv}^2(M_d, G)$, $f(g(x)) = f(x)$ for every $x \in M_d, g \in G$.
Assume we have a continuous positive definite kernel function $\mathbb{K}: M_d \times M_d \rightarrow \real_{+}$ on the data manifold $M_d$.
Let $\Set{F}(M_d) \subseteq L^2(M_d)$ denote the Reproducing Kernel Hilbert Space (RKHS) induced by $\kappa$, and $\Set{F}^{s}(M_d)$ denotes a subclass of $\Set{F}(M_d)$ including functions with square-integrable derivatives up to order $s > \dim(M_d) / 2$.
Denote $\Set{F}_{inv}^d(M_d) = \Set{F}^{d}(M_d) \cap L_{inv}^2(M_d, G)$, which include all $G$-invariant functions in $\Set{F}^{d}(M_d)$.
We refer readers to \cite{tahmasebi2023exact} for a detailed introduction to these mathematical notions.

\paragraph{Problem Setup.}
Now we introduce a data generation process.
We sample i.i.d. data points $\{x_1, \dots, x_n\}$ from a uniform distribution defined over $M_d$, and we synthesize the noisy labels via a ground-truth mapping $f^* \in \Set{F}^{s}(M_d)$: $x_i \sim \mathrm{Unif}(M_d)$ and $y_i = f^*(x_i) + \epsilon_i$, where $\epsilon_i$ is a random variable with $\mean[\epsilon_i | x_i] = 0$ and variance $\mean[\epsilon_i^2 | x_i] = \epsilon^2$.
We consider minimizing an $L_2$-loss (empirical risk) over all functions in a hypothesis class $f \in \Set{F}^s(M_d)$:
\begin{align*}
\wh{H}_{\eta}(f) := \frac{1}{n} \sum_{i=1}^{n} (y_i - f(x_i))^2 + \eta \lVert f \rVert_{\Set{F}(M_d)},
\end{align*}
where $\eta \ge 0$ is a coefficient to regularize the norm of $f$.
We denote the empirical risk minimization as $\hat{f} = \argmin_{f \in \Set{F}^d(M_d)} \wh{H}_{\eta}(f)$.
Moreover, we consider population risk over hypothesis $f$: $H(f) = \mean_{x, y}[(y - f(x))^2]$, where $x, y$ is sampled through our data generation process.

\begin{theorem}
Let $s = (1 + \kappa) \frac{d'}{2}$ with $\kappa > 0$ being a positive integer and $d' = \dim(M_d / G)$.
Consider $f^* \in \Set{F}_{inv}^{\theta s}(M_d)$ with $ \lVert f^* \rVert_{\Set{F}_{inv}^{\theta s}} = 1$ for $\theta \in (0, 1]$ and a set of $n$ data points generated by $f^*$. Suppose $\hat{f} = \argmin_{f \in \Set{F}^s(M_d), \lVert f \rVert_{\Set{F}_{inv}^s} = 1} \wh{H}_{\eta}(f)$ with $\eta$ chosen optimally to minimize the excessive risk, then with probability at least $1 - \delta$, for every $\epsilon > 0$, we have $H(\hat{f}) - H(f^*) \le \epsilon$ if:
\begin{align*}
n = \Theta\left(\max\left\{ \frac{\sigma^2 \vol(M_d / G)}{\epsilon^{1+1/\theta(\kappa+1)}} , \frac{\log(1/\delta)}{\epsilon^2}\right\}\right).
\end{align*}
Furthermore, if $G$ is finite, then
\begin{align*}
n = \Theta\left(\max\left\{ \frac{\sigma^2 \vol(M_d)}{|G| \epsilon^{1+1/\theta(\kappa+1)}} , \frac{\log(1/\delta)}{\epsilon^2}\right\}\right).
\end{align*}
\end{theorem}
\begin{proof}
By Theorem 4.1 of \cite{tahmasebi2023exact}, we bound the excessive risk as:
\begin{align*}
\mean\left[ H(\hat{f}) - H(f^*) \right] &\le 32\left( \frac{1}{\kappa \theta} \frac{\vol(\mathbb{B}^{d'})}{(2\pi)^{d'}} \frac{\sigma^2 \vol(M_d / G)}{n} \right)^{\theta s / (\theta s + d' / 2)} \left\lVert f^* \right\rVert_{\Set{F}_{inv}^{s\theta}}^{d' / (\theta s + d' / 2)} \\
&\overset{(*)}{\le} 32 \left( \frac{\sigma^2 \vol(M_d / G)}{n} \right)^{\frac{1}{\theta s + d'/2}},
\end{align*}
where $\mathbb{B}^{d'}$ denotes the $d'$-dimensional unit ball, $d' = \dim(M_d / G)$, $\theta \in (0, 1]$, and $s = \frac{(1 + \kappa)}{2}d'$ for some positive integer $\kappa > 0$.
In Eq. $(*)$, we notice that $\vol(\mathbb{B}^{d'}) / (2\pi)^{d'} = 2^{-d'} \pi^{-d'/2} \Gamma(d'/2 + 1)^{-1} \le 1$ and omit the function norm by the constraint $\left\lVert f^*\right\rVert_{\Set{F}_{inv}^{s}} = 1$.
Since $\hat{f}$ and $f^*$ both have norm $\lVert \hat{f} \rVert_{\Set{F}_{inv}^{\theta s}} = \lVert f^* \rVert_{\Set{F}_{inv}^s} = 1$, $H(\hat{f})$ satisfies bounded difference property $c_1 / n$ for some kernel-dependent constant $c_1 > 0$ \citep{bousquet2002stability}.
By McDiarmid's inequality, we have that, with probability at least $1-\delta$ \citep{shalev2014understanding},
\begin{align*}
H(\hat{f}) - H(f^*) \le \mean\left[H(\hat{f}) - H(f^*)\right] + \sqrt{\frac{c_1^2 \log(1/\delta)}{2n}},
\end{align*}
for some $c_1 > 0$.
Therefore, to achieve error less than $\epsilon$, we set $n$ as below:
\begin{align*}
n = \Theta\left(\max\left\{ \frac{\sigma^2 \vol(M_d / G)}{\epsilon^{(\theta s + d'/2) / \theta s}} , \frac{\log(1/\delta)}{\epsilon^2}\right\}\right) = \Theta\left(\max\left\{ \frac{\sigma^2 \vol(M_d / G)}{\epsilon^{1+1/\theta(\kappa+1)}} , \frac{\log(1/\delta)}{\epsilon^2}\right\}\right)
\end{align*}
We further note that, for finite $G$, $\vol(M_d / G) = \vol(M_d) / |G|$. This first term in maximum reduces to $\sigma^2 \vol(M_d) / (|G| \epsilon^{1+1/\theta(\kappa+1)})$, which finishes the proof.
\end{proof}

\begin{remark}
When we need the target function $f^*$ to be a continuous function on $M_d / G$, then $\theta s > d' / 2$. This implies $\theta (\kappa + 1) > 1$, and the sample complexity depends on $1/\epsilon$ quadratically.
\end{remark}